\documentclass[10pt,a4paper,logo]{googledeepmind}

\usepackage{mathtools}
\usepackage{CJKutf8}
\usepackage{multirow}
\usepackage{float}
\usepackage{natbib}
\usepackage{listings}
\usepackage{multicol}
\usepackage{tikz}
\usetikzlibrary{arrows.meta,positioning,fit,calc}
\usepackage{etoolbox}

\newtheorem{proposition}{Proposition}[section]

\usepackage{url}
\usepackage[normalem]{ulem}
\usepackage{longtable}
\usepackage{algorithm,algpseudocode}
\setcitestyle{authoryear,round}
\algrenewcommand{\algorithmicrequire}{\textbf{Input:}}
\algrenewcommand{\algorithmicensure}{\textbf{Output:}}
\newtheorem{lemma}[proposition]{Lemma}

\usepackage{amsmath,amsfonts,bm}

\def\eqref#1{equation~\ref{#1}}

\def\1{\bm{1}}

\DeclareMathAlphabet{\mathsfit}{\encodingdefault}{\sfdefault}{m}{sl}
\SetMathAlphabet{\mathsfit}{bold}{\encodingdefault}{\sfdefault}{bx}{n}

\AtBeginDocument{\setlength{\emergencystretch}{3em}}

\let\PtwoOriginalTable\table
\let\PtwoOriginalEndTable\endtable
\renewenvironment{table}[1][!htbp]
  {\PtwoOriginalTable[!htbp]}
  {\PtwoOriginalEndTable}
\usepackage{placeins}

\newcommand{\blockmethod}{P2BPO}

\newcommand{\meanstd}[2]{\ensuremath{#1{\scriptstyle\,\pm\,#2}}}

\graphicspath{{images/}{figures/}}
\hypersetup{
  pdftitle={P2TPO: Structure-Aligned Page-to-Token Credit Assignment for Full-Page Document Parsing},
  pdfauthor={Yong Zu, Ninghua Yang, Zezhong Tan, Xinhong Ma, Ziqiang Dong, Yang Yang}
}

\title{P2TPO: Structure-Aligned Page-to-Token Credit Assignment for Full-Page Document Parsing}
\author[1,2]{Yong Zu\textsuperscript{\ensuremath{\S\ddagger}}}
\author[1]{Ninghua Yang\textsuperscript{\ensuremath{\S}}}
\author[1]{Zezhong Tan}
\author[1]{Xinhong Ma}
\author[1]{Ziqiang Dong\textsuperscript{\ensuremath{\ast}}}
\author[1]{Yang Yang\textsuperscript{\ensuremath{\dagger}}}
\affil[1]{Qwen Business Unit, Alibaba}
\affil[2]{Xidian University}
\footnotetext{\textsuperscript{\ensuremath{\S}}Core contributors.\quad
  \textsuperscript{\ensuremath{\ast}}Project lead.\quad
  \textsuperscript{\ensuremath{\dagger}}Team lead.\quad
  \textsuperscript{\ensuremath{\ddagger}}Work done during an internship at Alibaba.}
\date{}

\AfterEndEnvironment{abstract}{\maketitle\suppressfloats[t]}
\pretocmd{\bibliography}{\FloatBarrier}{}{}
\BeforeBeginEnvironment{tabularx}{\ifvmode\noindent\fi}

\begin{document}
\begin{CJK*}{UTF8}{gbsn}
\begin{abstract}
End-to-end full-page document parsing must preserve the content and structure of heterogeneous elements, yet otherwise accurate outputs can contain local errors and omissions. A single page-level advantage does not explicitly identify which elements differ in quality across sampled outputs or where their errors occur. We introduce Page-to-Token Policy Optimization (P2TPO), a structure-aligned credit-assignment framework that operates on full-page policy trajectories. Reference-element alignment first associates block-level quality contrasts with generated spans. Unified Anchor Pricing (UAP) incorporates matched and missing element outcomes and provides feedback relative to a fixed quality ceiling when sampled outputs receive identical but imperfect scores for a reference element. Within matched blocks, Span Credit Redistribution (SCR) redistributes existing block credit according to aligned correct and erroneous spans, using boundary proxies for within-block omissions while preserving the total block-derived advantage across its tokens. Under a fixed-weight additive page utility, we show that element-indexed contrasts retain local quality differences that page-level aggregation can cancel out. In the main experimental comparison, P2TPO outperforms all evaluated policy optimization baselines in seed-averaged scores on each of the six document parsing benchmarks. Additional experiments demonstrate gains across model scales and initializations.
\end{abstract}

\section{Introduction}

End-to-end document parsers transcribe images of digital documents, scans, and photographs into a single structured sequence in reading order~\citep{blecher2023nougat,taghadouini2026lightonocr,ouyang2024omnidocbench,du2025docptbench}. We study full-page Markdown transcription, with faithful text, \LaTeX{} formulas, and HTML tables. Dense heterogeneous content, skew, warping, and uneven illumination challenge accuracy and completeness~\citep{ouyang2024omnidocbench,zhou2026real5}. Errors are often local: an otherwise accurate output may corrupt a formula, omit words, or miss an entire element.

Supervised next-token learning serves as the foundation for these parsers, while reinforcement learning (RL) is used for post-training optimization of reference-based measures for fully generated documents. These measures evaluate content accuracy, structural fidelity, and completeness. Recent parsers highlight the effectiveness of reinforcement learning in post-training through layout-aware rewards and verifiable unit tests ~\citep{wang2025infinityparser,chen2025logicsparsing,poznanski2025olmocr2}. A key question for full-page transcription is how to convert these measurements into feedback that corrects missing or corrupted content while preserving correctly transcribed sections.

With page-level outcome rewards, GRPO assigns one response advantage to all generated tokens~\citep{shao2024deepseekmath}. Yet two outputs of the same page can receive similar overall scores while making opposite local errors: one transcribes a formula correctly but omits a table, whereas the other recovers the table and corrupts the formula (Figure~\ref{fig:overview}(a)). A shared response coefficient does not preserve these element-specific contrasts. GDPO normalizes reward dimensions separately before combining them into a shared response coefficient~\citep{liu2026gdpo}; CSPO routes table-wide functional-attribute advantages to corresponding tokens~\citep{yang2026cspo}. Our focus is the remaining reference-instance question: how to compare corresponding individual elements across full-page outputs, include missing outcomes, and allocate their signed credit to generated positions.

We analyze this phenomenon at three levels. First, reference identity maintains contrasts between individual elements, including instances of the same type. Second, element-wise comparisons must account for missing outcomes and retain feedback when all outcomes are equally imperfect: an omitted element lacks a generated span, and centering tied scores diminishes its relative signal. Third, feedback must differentiate erroneous positions from correct content within a matched block, rather than assigning uniform credit across the block. Together, the four failure modes illustrated in Figure~\ref{fig:overview}(b) highlight three key requirements: preserve element identity, complete the comparison outcomes, and control where credit is assigned.

\begin{figure}[!htbp]
\centering
\includegraphics[width=\linewidth]{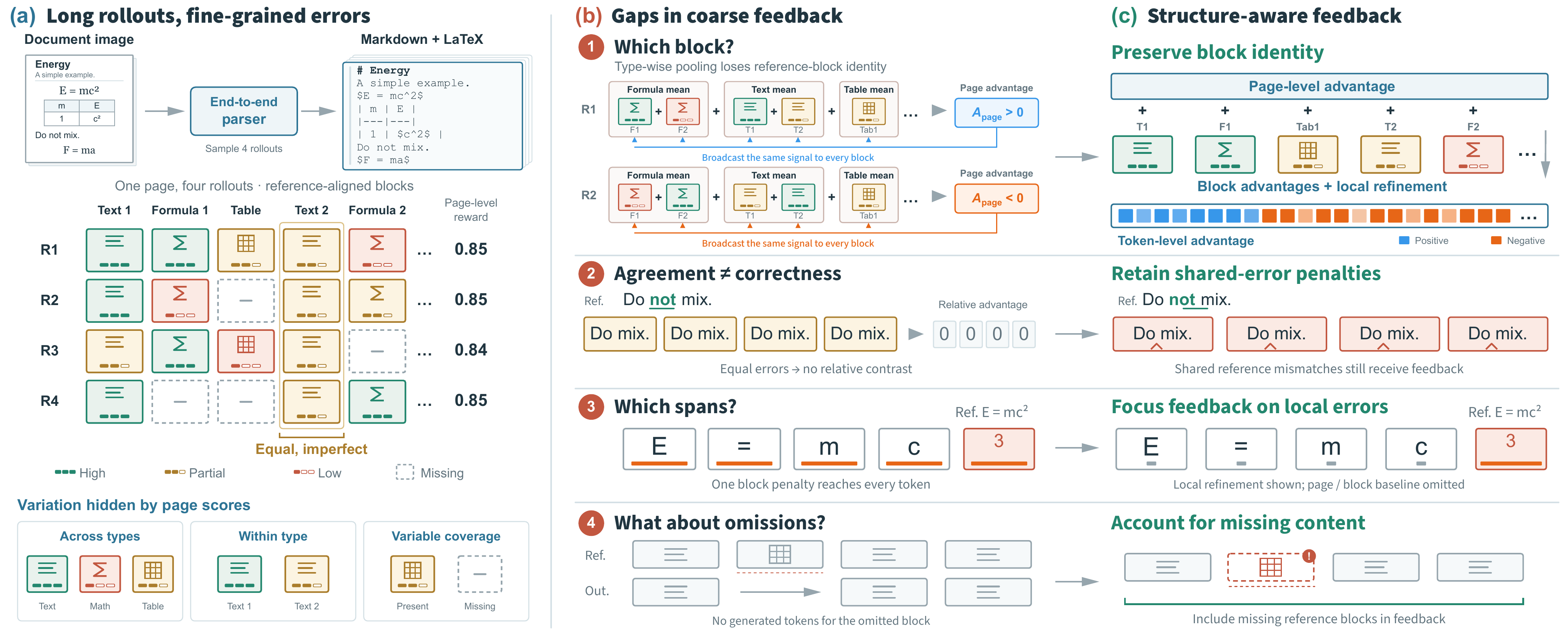}%
\caption{Schematic motivation for structure-aware feedback. (a) Similar page rewards hide local differences. (b) Coarse feedback obscures block identity, shared errors, and error/omission locations. (c) Structure-aware feedback retains element identity and assigns local credit. Rewards and serialized content are illustrative; training outputs use HTML tables.}
\label{fig:overview}
\end{figure}

We introduce \emph{Page-to-Token Policy Optimization} (P2TPO), which assigns local credit while retaining full-page generation and page-level feedback (Figure~\ref{fig:overview}(c)). It first aligns each output to reference elements so that corresponding instances can be compared across rollouts. \emph{Unified Anchor Pricing} (UAP) completes these comparisons with missing outcomes and uses a quality ceiling to retain feedback at imperfect consensus. Matched coefficients act on generated blocks; whole-element omissions act through existing boundary tokens. \emph{Span Credit Redistribution} (SCR) then reallocates an eligible matched block's coefficient toward aligned errors when negative and correct positions when positive, preserving its sign and total credit mass. The final update combines this local feedback with the shared page term.

Our contributions are fourfold:
\begin{itemize}
    \item \textbf{Reference-element credit assignment.} We retain individual element identities across full-page rollouts and characterize the local quality contrasts that aggregation can erase under a fixed-weight additive page utility.
    \item \textbf{Complete outcome feedback.} We propose the UAP, which outlines a unified set of rules for handling matches, managing missing values, and addressing uniformly imperfect outcomes. This includes a feedback mechanism for instances where an entire element is absent.
    \item \textbf{Budget-preserving intra-block redistribution.} We developed the SCR mechanism, which reallocates block credits while preserving both the sign and the total amount of matched block credits. We proved the conservation of credits and specified the conditions under which informative masks can increase the allocation to the target position.
    \item \textbf{Empirical effectiveness and generality.} Using a shared supervised checkpoint, matched update budgets, and a five-seed run, P2TPO outperforms four response-level baselines on all 13 primary splits across six benchmarks that include digital documents, scanned pages, and real-world photographs. In addition, across three different model and initialization settings, it achieves the highest mean in 37 out of 39 comparisons of splits and settings.
\end{itemize}

\section{Method}
\label{sec:method}

Given a page image, its reference document, and a set of sampled full-page outputs, P2TPO assigns a single coefficient to each generated token. Reference alignment identifies corresponding elements across outputs. UAP converts matched and missing outcomes into element coefficients; SCR redistributes eligible matched coefficients within their generated blocks, while whole-element omissions use a separate boundary route. These local terms combine with a shared page coefficient and enter the policy update. Generation and policy updates retain full-page autoregressive conditioning.

\begin{figure}[t]
\centering
\includegraphics[width=\linewidth]{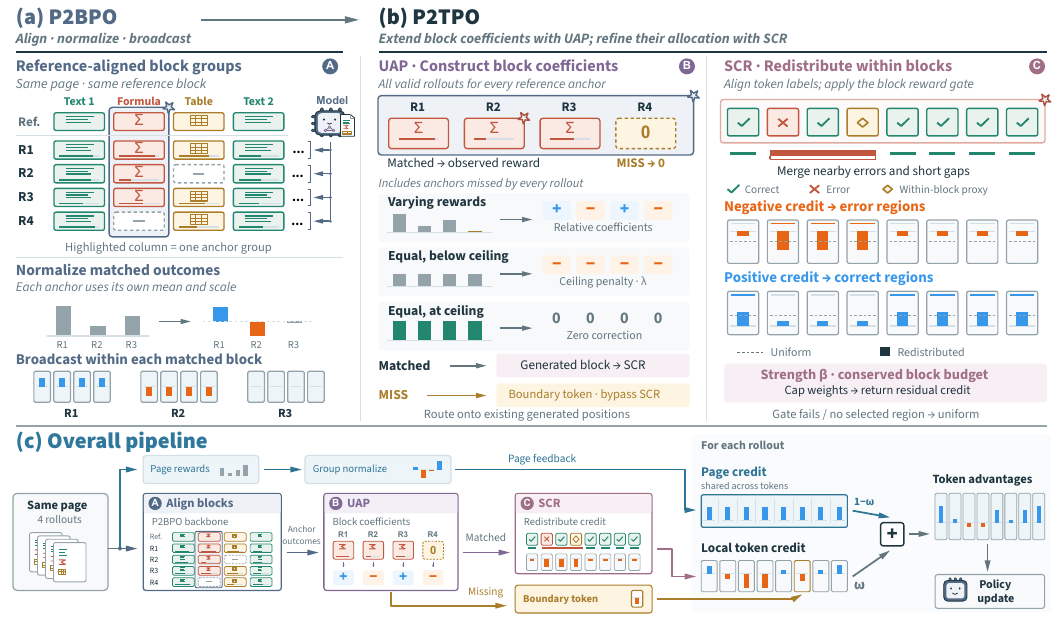}
\caption{P2TPO overview. (a) P2BPO aligns reference elements, normalizes matched outcomes, and broadcasts block credit. (b) UAP incorporates missing outcomes and imperfect-consensus feedback; SCR redistributes matched-block credit under a conserved budget. (c) The complete pipeline combines local token credit with page feedback. Missing-block coefficients reach boundary tokens directly, bypassing SCR.}
\label{fig:method-overview}
\end{figure}

\subsection{Reference-element comparisons}
\label{sec:alignment}
\label{sec:block}

Given a page image $x$ and prompt $p$, we follow GRPO's group sampling protocol~\citep{shao2024deepseekmath} and draw $G$ full-page outputs $y_i\sim\pi_{\theta_{\mathrm{old}}}(\cdot\mid x,p)$. Predictions and references are parsed into ordered text, formula, and table blocks. Content-based sequence alignment permits insertions and deletions while preserving reading order. We index each matched prediction by its corresponding reference block $y_k^\star$, rather than its position in the generated sequence. Thus, different instances of the same category remain distinct across rollouts, even when preceding elements are omitted.

Let $i$, $k$, and $t$ index rollouts, reference elements, and tokens. A matched block has a type-specific quality score $s_{i,k}\in[0,c_k]$ and a generated-token support $\mathcal B_{i,k}$ of length $L_{i,k}$. The implementation uses normalized scores with $c_k=1$ and maps generated block boundaries to disjoint token spans. For each reference $k$ with a valid match, let $\mathcal I_k$ contain its valid matched rollouts, and let $\mu_k^{\mathrm B}$ and $\sigma_k^{\mathrm B}$ be the mean and population standard deviation of their scores.

\emph{Page-to-Block Policy Optimization} (\blockmethod{}) normalizes each matched reference group independently:
\begin{equation}
A_{i,k}^{\mathrm B}=\begin{cases}
(s_{i,k}-\mu_k^{\mathrm B})/\sigma_k^{\mathrm B},&\sigma_k^{\mathrm B}>\epsilon,\\[2pt]
0,&\sigma_k^{\mathrm B}\le\epsilon,
\end{cases}
\qquad i\in\mathcal I_k.
\label{eq:p2bpo}
\end{equation}
Here $\epsilon\ge0$ is a numerical tolerance. Each coefficient is broadcast over $\mathcal B_{i,k}$ (see Fig.~\ref{fig:method-overview}(a)), so blocks within one response can receive different signs and magnitudes. Reference-wise comparisons retain contrasts that page- or category-level pooling can cancel. However, unmatched references have no coefficient, while a singleton or tied matched group produces zero credit. UAP addresses these two gaps. Appendix~\ref{app:alignment-details} summarizes the algorithm inputs.

\subsection{UAP: complete outcomes and retain quality deficits}
\label{sec:uap}

\emph{Unified Anchor Pricing} (UAP) completes each reference element's outcomes across valid rollouts $\mathcal V_k$: $Q_{i,k}=s_{i,k}$ when matched and $Q_{i,k}=0$ when missing. With mean $\mu_k$ and population standard deviation $\sigma_k$, its coefficient uses relative comparison when outcomes differ and ceiling-gap feedback when dispersion is negligible:
\begin{equation}
A_{i,k}^{\mathrm U}=\begin{cases}
(Q_{i,k}-\mu_k)/\sigma_k,&\sigma_k>\epsilon,\\[2pt]
-\lambda h(\mu_k,c_k),&\sigma_k\le\epsilon,
\end{cases}
\qquad h(u,c)=\max(0,c-u).
\label{eq:uap}
\end{equation}
Here $\lambda\ge0$ weights the ceiling gap; the same tolerance $\epsilon$ separates the two branches. For $\lambda>0$, both fully matched imperfect ties and all-missing groups receive negative coefficients, whereas a ceiling-quality tie receives zero. Thus a single element-indexed rule retains feedback for missing and uniformly imperfect outcomes.

In a mixed group, zero-valued omissions and observed block scores jointly determine the mean and dispersion, so outcome completion can change coefficients for both matched and missing outcomes.

Outcome completion includes missing scores in the reference-wise comparison. Matched coefficients act on generated blocks, while missing-element coefficients are routed to existing boundary tokens through $\nu_{i,k}(t)\ge0$, where $\sum_t\nu_{i,k}(t)=1$, bypassing SCR. Appendix~\ref{app:credit-pseudocode} details how we construct the coefficients.

\subsection{SCR: redistribute credit under a conserved budget}
\label{sec:scr}

\emph{Span Credit Redistribution} (SCR) reallocates an available block coefficient under a conserved budget, concentrating negative credit in error regions and positive credit in correct regions.

Quality gating restricts redistribution to high-scoring, type-matched blocks, focusing local allocation on residual defects within otherwise accurate elements. For eligible matched blocks, SCR creates edit-based masks based on both the prediction and the reference. Insertions and replacements identify error locations. SCR assigns deleted reference content within a matched block to a nearby generated position. Whole-element omissions follow UAP's separate boundary route. Further details appear in Appendix~\ref{app:credit-specification}. A binary mask $u_{i,t}$ selects errors for negative credit and correct positions for positive credit (see Fig.~\ref{fig:method-overview}(b)). For a nonempty mask of size $N_{i,k}=\sum_{t\in\mathcal B_{i,k}}u_{i,t}$, the redistribution has the form
\begin{equation}
d_{i,t}=1-\eta_{i,k}+\eta_{i,k}\frac{L_{i,k}u_{i,t}}{N_{i,k}},
\qquad \sum_{t\in\mathcal B_{i,k}}d_{i,t}=L_{i,k}.
\label{eq:scr}
\end{equation}
The implementation mixes uniform and selected-token allocation, applies a pre-rebate cap to concentrated weights, and returns the removed mass uniformly. This determines the effective strength $\eta_{i,k}\in[0,1]$ in Appendix~\ref{app:scr-procedure}. Nonnegative weights preserve the coefficient's sign, while their unit mean preserves both the signed and absolute block-credit budgets. Blocks ineligible for redistribution or with empty masks retain uniform weights \(d_{i,t} = 1\). Appendix~\ref{subsec:span-scr} provides details on the resulting increase in target allocation for informative masks.

\subsection{Page-to-token composition}
\label{sec:objective}

Let $\mathcal K_i$ and $\mathcal M_i$ index matched and missing reference elements, respectively. For each rollout, we clip raw page-reward dimensions to $[0,1]$ and average over the available dimensions. Normalizing these scalar rewards across rollouts for the same prompt gives $A_i^{\mathrm P}$ (Appendix~\ref{app:credit-specification}). Combining the two local contributions with page feedback gives
\begin{equation}
\begin{aligned}
C_{i,t}^{\mathrm{US}}&=
\sum_{k\in\mathcal K_i}A_{i,k}^{\mathrm U}\mathbf1[t\in\mathcal B_{i,k}]d_{i,t}
+\sum_{k\in\mathcal M_i}A_{i,k}^{\mathrm U}\nu_{i,k}(t),\\
A_{i,t}^{\mathrm{US}}&=(1-\omega)A_i^{\mathrm P}+\omega C_{i,t}^{\mathrm{US}},
\qquad \omega\in[0,1].
\end{aligned}
\label{eq:composition}
\end{equation}
The page term is shared across generated positions, while local terms vary with reference identity and within-block location. Tokens outside local supports retain the page term. The mixing weight $\omega$ balances these feedback scopes before the policy update. Alignments and the resulting $A_{i,t}^{\mathrm{US}}$ remain fixed during the clipped policy update; token probabilities retain the full autoregressive page context. Appendix~\ref{app:credit-pseudocode} provides the complete coefficient construction, component variants, and policy loss; Appendix~\ref{app:training-config} lists parameter settings.

\section{Related Work}

\paragraph{Document parsing and reinforcement post-training.}
End-to-end parsers such as Donut and Nougat generate structured documents directly from images~\citep{kim2022donut,blecher2023nougat}. Recent systems extend this setting to diverse layouts and capture conditions~\citep{cai2026navidcocr,lu2026ovisocr2}. Reinforcement post-training introduces rewards for transcription, layout, reading order, output validity, and structural fidelity~\citep{wang2025infinityparser,huang2026infinityparser2,chen2025logicsparsing,poznanski2025olmocr2,wang2026texocr,hunyuan2025ocr,zhang2026paddleocrvl16}. Table- and element-specific approaches further tailor the prediction unit and reward to local content~\citep{zhang2025monkeyocr15,li2026structab,wang2026docpo}. These systems motivate heterogeneous quality measures, while P2TPO studies how feedback for individual reference elements reaches tokens within a full-page rollout. Appendix~\ref{app:reward-interfaces} compares reward units and optimization supports.

\paragraph{Group-relative optimization and reward normalization.}
GRPO and RLOO construct response advantages from sampled groups without a learned critic~\citep{shao2024deepseekmath,ahmadian2024rloo}. GSPO uses sequence-level importance weighting and clipping, while DAPO and Dr.~GRPO revise sampling, clipping, reduction, or normalization~\citep{zheng2025gspo,yu2025dapo,liu2025drgrpo}. GDPO, RDPO, and GD$^2$PO address the normalization and interaction of multiple reward dimensions~\citep{liu2026gdpo,bai2026rdpo,liu2026gd2po}. NGRPO combines virtual-maximum calibration with asymmetric clipping to learn from homogeneous failures~\citep{nan2025ngrpo}. UAP shares ceiling-informed feedback, while unifying missing outcomes and imperfect reward consensus within reference-element comparisons and routing their feedback to generated tokens. Appendix~\ref{app:comparison-methods} details the evaluated baselines; Appendix~\ref{app:reward-interfaces} compares reward interfaces.

\paragraph{Structured and local credit assignment.}
Process supervision evaluates reasoning steps~\citep{lightman2023verify,wang2024mathshepherd}; GiGPO compares repeated environment states while retaining trajectory feedback~\citep{feng2025gigpo}. CSPO extends GRPO with table-wide functional-attribute rewards whose normalized advantages are routed to corresponding tokens alongside global feedback~\citep{yang2026cspo}. P2TPO performs local comparisons for individual reference-block instances in full-page rollouts, preserving quality differences between elements of the same type and supporting omission feedback and intra-block redistribution. Functional-attribute assessment and reference-instance comparison are complementary; Appendix~\ref{app:related-methods} discusses possible adaptations.

\section{Experiments}
\label{sec:experiments}

Unless otherwise specified, ``seeds'' refers to distinct RL training runs rather than repeated evaluations of a fixed checkpoint. The reported means and standard deviations summarize benchmark scores across these training runs.
The experiments test whether reference-element feedback improves full-page parsing, whether UAP and SCR contribute to those gains, and whether alignment-informed placement matters beyond nonuniform weighting. Additional settings assess applicability across initializations, model scales, and backbones; parameter sweeps and a 500-update study examine sensitivity to credit parameters and training horizon.

\subsection{Experimental setup}
\label{sec:setup}

\paragraph{Baselines.}
We compare P2TPO with four response-level policy optimization baselines: GRPO~\citep{shao2024deepseekmath}, GSPO~\citep{zheng2025gspo}, GDPO~\citep{liu2026gdpo}, and RDPO~\citep{bai2026rdpo}. \blockmethod{} and its +UAP and +SCR variants serve as component ablations. Appendix~\ref{app:comparison-methods} details the comparison methods.

\paragraph{Benchmarks.}
We evaluate on six benchmarks with 13 splits: OmniDocBench v1.6~\citep{ouyang2024omnidocbench,omnidoc2026release} for heterogeneous full-page parsing; PureDocBench~\citep{li2026puredocbench} for Clean, Digital, and Real conditions; FinixDocBench~\citep{wang2026finixdoc} for digital and photographed documents; Real5-OmniDocBench~\citep{zhou2026real5} for five capture conditions; DocPTBench~\citep{du2025docptbench} for photographed and dewarped pages; and Wild-OmniDocBench~\citep{li2026wildomni} for real-world captures. Scores compare methods within each benchmark: five benchmarks use the shared OmniDocBench evaluator, while FinixDocBench pools its Digital and Photo tracks using its official evaluator. Real5 and Wild reuse OmniDocBench source pages and assess capture robustness. Appendix~\ref{app:external} specifies the evaluated scope and score definitions, including the adapted DocPTBench composite.

\paragraph{Implementation details.}
We use converted Infinity-Doc2-5M annotations~\citep{huang2026infinityparser2} for document SFT followed by RL. SFT provides the initial checkpoint: Qwen3.5-2B trained on 700K pages in the primary setting, with additional settings covering 50K-page checkpoints, Qwen3.5-4B, and Qwen3-VL-2B-Instruct. Within each setting, all RL methods start from the same checkpoint and share the update budget. Init reports its measured benchmark score before RL. Table~\ref{tab:all-benchmarks} reports primary RL Overall scores as means $\pm$ standard deviations over five seeds. Appendices~\ref{app:training-data}, \ref{app:training-config}, and~\ref{app:generality-efficiency} provide further details.

\subsection{Cross-benchmark performance}
\label{sec:cross-benchmark}

Table~\ref{tab:all-benchmarks} reports Overall scores for the primary Qwen3.5-2B/700K setting across all six benchmarks and 13 evaluation splits.
\begin{table}
\centering
\caption{Cross-benchmark Overall scores ($\uparrow$, 0--100) from Qwen3.5-2B initialized by SFT on 700K pages. Init reports pre-RL scores; RL entries report mean $\pm$ standard deviation over five seeds after 100 updates.}
\label{tab:all-benchmarks}
\label{tab:results}
\scriptsize
\setlength{\tabcolsep}{0.8pt}
\renewcommand{\arraystretch}{1.08}
\resizebox{\linewidth}{!}{%
\begin{tabular}{@{}l@{\hspace{3pt}}l@{\hspace{3pt}}c@{\hspace{3pt}}*{8}{c}@{}}
\toprule
Benchmark & Split & Init & GRPO & GSPO & GDPO & RDPO & P2BPO & +UAP & +SCR & \textbf{P2TPO} \\
\midrule
\textbf{OmniDocBench} & v1.6 & 92.17 & \meanstd{92.25}{0.34} & \meanstd{92.38}{0.18} & \meanstd{92.31}{0.24} & \meanstd{92.32}{0.35} & \meanstd{92.48}{0.15} & \meanstd{\underline{92.84}}{0.30} & \meanstd{92.56}{0.21} & \meanstd{\mathbf{93.12}}{0.23} \\
\cmidrule(lr){1-11}
\multirow{3}{*}{\textbf{PureDocBench}} & Clean & 81.28 & \meanstd{83.19}{0.70} & \meanstd{83.44}{0.31} & \meanstd{82.68}{0.65} & \meanstd{83.10}{0.54} & \meanstd{\underline{84.16}}{0.62} & \meanstd{83.98}{0.37} & \meanstd{83.78}{0.40} & \meanstd{\mathbf{84.68}}{0.74} \\
 & Digital & 76.66 & \meanstd{79.89}{0.48} & \meanstd{79.50}{0.33} & \meanstd{79.24}{0.50} & \meanstd{79.26}{0.50} & \meanstd{\underline{80.31}}{0.48} & \meanstd{80.26}{0.39} & \meanstd{79.79}{0.31} & \meanstd{\mathbf{81.03}}{0.60} \\
 & Real & 63.29 & \meanstd{68.89}{0.64} & \meanstd{67.64}{0.38} & \meanstd{68.23}{0.69} & \meanstd{68.27}{0.28} & \meanstd{69.36}{0.46} & \meanstd{\underline{69.87}}{0.49} & \meanstd{69.08}{0.35} & \meanstd{\mathbf{70.33}}{0.70} \\
\cmidrule(lr){1-11}
\textbf{FinixDocBench} & --- & 65.30 & \meanstd{69.36}{0.49} & \meanstd{69.24}{0.44} & \meanstd{69.21}{1.42} & \meanstd{68.15}{0.74} & \meanstd{70.03}{0.77} & \meanstd{\underline{70.29}}{0.62} & \meanstd{69.37}{0.44} & \meanstd{\mathbf{70.92}}{0.92} \\
\cmidrule(lr){1-11}
\multirow[c]{5}{*}{\shortstack[l]{\textbf{Real5-}\\\textbf{OmniDocBench}}} & Scanning & 91.06 & \meanstd{92.07}{0.20} & \meanstd{92.00}{0.17} & \meanstd{92.00}{0.21} & \meanstd{91.90}{0.28} & \meanstd{\underline{92.44}}{0.29} & \meanstd{\mathbf{92.52}}{0.16} & \meanstd{92.25}{0.31} & \meanstd{92.40}{0.12} \\
 & Warping & 86.50 & \meanstd{90.45}{0.53} & \meanstd{89.36}{0.29} & \meanstd{90.11}{0.15} & \meanstd{89.92}{0.37} & \meanstd{90.44}{0.57} & \meanstd{\underline{90.59}}{0.30} & \meanstd{90.31}{0.46} & \meanstd{\mathbf{90.77}}{0.45} \\
 & Skew & 81.77 & \meanstd{88.26}{0.48} & \meanstd{88.47}{0.34} & \meanstd{88.09}{0.39} & \meanstd{88.47}{0.42} & \meanstd{\underline{88.82}}{0.64} & \meanstd{88.78}{0.38} & \meanstd{88.56}{0.61} & \meanstd{\mathbf{89.15}}{0.52} \\
 & Illumination & 91.23 & \meanstd{92.41}{0.12} & \meanstd{91.98}{0.26} & \meanstd{92.22}{0.20} & \meanstd{92.35}{0.09} & \meanstd{92.59}{0.20} & \meanstd{\underline{92.62}}{0.19} & \meanstd{92.41}{0.28} & \meanstd{\mathbf{92.83}}{0.17} \\
 & Screen photo & 90.32 & \meanstd{92.04}{0.17} & \meanstd{91.76}{0.31} & \meanstd{91.75}{0.30} & \meanstd{91.52}{0.28} & \meanstd{\underline{92.33}}{0.30} & \meanstd{92.28}{0.12} & \meanstd{92.20}{0.50} & \meanstd{\mathbf{92.62}}{0.29} \\
\cmidrule(lr){1-11}
\multirow{2}{*}{\textbf{DocPTBench}} & SynReal & 62.27 & \meanstd{73.21}{0.67} & \meanstd{72.85}{0.65} & \meanstd{72.59}{0.54} & \meanstd{72.80}{0.26} & \meanstd{73.61}{0.56} & \meanstd{73.60}{0.45} & \meanstd{\mathbf{73.81}}{1.14} & \meanstd{\underline{73.79}}{0.39} \\
 & Dewarp & 79.66 & \meanstd{83.57}{0.29} & \meanstd{82.85}{0.36} & \meanstd{83.12}{0.30} & \meanstd{83.29}{0.82} & \meanstd{\underline{83.95}}{0.44} & \meanstd{\mathbf{83.99}}{0.42} & \meanstd{83.36}{0.50} & \meanstd{83.76}{0.51} \\
\cmidrule(lr){1-11}
\textbf{Wild-OmniDocBench} & --- & 80.04 & \meanstd{84.67}{0.18} & \meanstd{83.48}{0.20} & \meanstd{84.28}{0.20} & \meanstd{84.50}{0.29} & \meanstd{\underline{84.72}}{0.47} & \meanstd{84.68}{0.54} & \meanstd{84.55}{0.56} & \meanstd{\mathbf{85.02}}{0.21} \\
\bottomrule
\end{tabular}%
}
\par\vspace{3pt}
{\scriptsize\raggedright
\textit{Note.} Init is the shared SFT-initialized checkpoint; no standard deviation is reported for Init. \blockmethod{} is the matched-element reference configuration; +UAP and +SCR enable one component at a time. Bold and underlined values mark the best and second-best RL means, respectively. --- denotes an aggregate benchmark score.\par}
\end{table}

On the primary OmniDocBench v1.6 benchmark, P2TPO reaches $93.12\pm0.23$, a gain of 0.95 points from its pre-RL score and 0.74 points over the strongest response-level baseline. It also exceeds the strongest response-level mean on all 12 external splits, with margins from 0.19 points on DocPTBench Dewarp to 1.56 points on FinixDocBench. The next-largest margins occur on PureDocBench Real (+1.44) and Clean (+1.24), and the gains hold across all five Real5-OmniDocBench capture conditions. 

\subsection{Component ablations and credit diagnostics}
\label{sec:complementarity}
\label{sec:credit-analysis}
\label{sec:qualitative}

\paragraph{Reference-element routing.}
Matched-element comparisons account for a substantial part of the gain: \blockmethod{} exceeds the strongest response-level baseline in mean Overall score on 12 of 13 splits, with the largest margins on PureDocBench Clean (+0.72 points) and FinixDocBench (+0.67 points; Table~\ref{tab:all-benchmarks}). All differences below are in mean Overall-score points.

\paragraph{SCR is more consistently beneficial with UAP.}
SCR benefits most consistently from UAP's completed and calibrated feedback. With UAP enabled, adding SCR improves the mean on 11 of 13 splits, led by PureDocBench Digital (+0.77) and Clean (+0.70); without UAP, it improves two of 13 splits. Adding UAP to \blockmethod{} + SCR improves 12 of 13 splits. Full P2TPO attains the highest mean among the four component configurations on ten splits (Table~\ref{tab:all-benchmarks}). These descriptive five-seed comparisons support the complementarity of feedback construction and token allocation.

\paragraph{Internal UAP interventions.}
We retain outcome completion and SCR while separately disabling consensus feedback ($\lambda=0$) or omission routing. Table~\ref{tab:mechanism-robustness}(a) shows that full P2TPO exceeds the former in mean score on eight of 13 splits and the latter on 11. Full P2TPO also has lower observed across-seed standard deviations than the $\lambda=0$ configuration on all 13 splits. These interventions assess UAP's feedback and delivery rules; Appendix~\ref{app:uap-ablations} gives their definitions and detailed comparisons.

\paragraph{SCR allocation controls.}
To isolate the contribution of token placement, we preserve each eligible block's coefficient values and selected-token count while randomizing their positions. Aligned SCR exceeds random support in mean score on 12 of 13 splits and circular shift on ten (Table~\ref{tab:mechanism-robustness}(b)). Random support and circular shift exceed uniform allocation on two and three splits, respectively. These comparisons support alignment-informed placement beyond nonuniform weighting alone; Appendix~\ref{app:scr-allocation-controls} gives the protocols and detailed results.

\begin{table}
\centering
\caption{UAP ablations and SCR allocation controls across six benchmarks and 13 splits. Entries report mean Overall score ($\uparrow$, 0--100) $\pm$ standard deviation over five runs; settings follow Table~\ref{tab:all-benchmarks}.}
\label{tab:mechanism-robustness}
\scriptsize
\setlength{\tabcolsep}{2pt}
\renewcommand{\arraystretch}{1.0}
\resizebox{\linewidth}{!}{%
\begin{tabular}{@{}l@{\hspace{3pt}}l@{\hspace{3pt}}c@{\hspace{3pt}}*{6}{c}@{}}
\toprule
\multirow{2}{*}{Benchmark} & \multirow{2}{*}{Split} & \multirow{2}{*}{\shortstack{Full\\P2TPO}} & \multicolumn{2}{c}{\textbf{(a) UAP interventions}} & \multirow{2}{*}{\shortstack{w/o\\UAP}} & \multicolumn{2}{c}{\textbf{(b) SCR allocation controls}} & \multirow{2}{*}{\shortstack{w/o\\SCR}} \\
\cmidrule(lr){4-5}\cmidrule(lr){7-8}
& & & \shortstack{w/o consensus\\feedback} & \shortstack{w/o omission\\routing} & & \shortstack{Circular\\shift} & \shortstack{Random\\support} & \\
\midrule
\textbf{OmniDocBench} & v1.6 & \meanstd{\mathbf{93.12}}{0.23} & \meanstd{92.92}{0.45} & \meanstd{92.86}{0.23} & \meanstd{92.56}{0.21} & \meanstd{92.73}{0.17} & \meanstd{92.82}{0.25} & \meanstd{92.84}{0.30} \\
\cmidrule(lr){1-9}
\multirow{3}{*}{\textbf{PureDocBench}} & Clean & \meanstd{\mathbf{84.68}}{0.74} & \meanstd{84.38}{0.94} & \meanstd{83.74}{0.35} & \meanstd{83.78}{0.40} & \meanstd{83.85}{0.79} & \meanstd{83.84}{0.16} & \meanstd{83.98}{0.37} \\
 & Digital & \meanstd{\mathbf{81.03}}{0.60} & \meanstd{80.79}{0.88} & \meanstd{80.15}{0.36} & \meanstd{79.79}{0.31} & \meanstd{80.08}{0.73} & \meanstd{80.43}{0.45} & \meanstd{80.26}{0.39} \\
 & Real & \meanstd{\mathbf{70.33}}{0.70} & \meanstd{70.05}{0.93} & \meanstd{69.67}{0.64} & \meanstd{69.08}{0.35} & \meanstd{69.69}{0.78} & \meanstd{69.33}{0.47} & \meanstd{69.87}{0.49} \\
\cmidrule(lr){1-9}
\textbf{FinixDocBench} & --- & \meanstd{70.92}{0.92} & \meanstd{70.42}{1.36} & \meanstd{70.05}{0.80} & \meanstd{69.37}{0.44} & \meanstd{\mathbf{70.98}}{1.22} & \meanstd{70.03}{1.09} & \meanstd{70.29}{0.62} \\
\cmidrule(lr){1-9}
\multirow[c]{5}{*}{\shortstack[l]{\textbf{Real5-}\\\textbf{OmniDocBench}}} & Scanning & \meanstd{92.40}{0.12} & \meanstd{92.48}{0.35} & \meanstd{92.45}{0.19} & \meanstd{92.25}{0.31} & \meanstd{92.26}{0.41} & \meanstd{92.36}{0.23} & \meanstd{\mathbf{92.52}}{0.16} \\
 & Warping & \meanstd{\mathbf{90.77}}{0.45} & \meanstd{90.68}{0.77} & \meanstd{90.45}{0.39} & \meanstd{90.31}{0.46} & \meanstd{90.41}{0.35} & \meanstd{90.25}{0.48} & \meanstd{90.59}{0.30} \\
 & Skew & \meanstd{\mathbf{89.15}}{0.52} & \meanstd{88.92}{0.81} & \meanstd{88.80}{0.31} & \meanstd{88.56}{0.61} & \meanstd{88.70}{0.46} & \meanstd{88.42}{0.28} & \meanstd{88.78}{0.38} \\
 & Illumination & \meanstd{92.83}{0.17} & \meanstd{\mathbf{92.84}}{0.38} & \meanstd{92.59}{0.16} & \meanstd{92.41}{0.28} & \meanstd{92.52}{0.45} & \meanstd{92.45}{0.16} & \meanstd{92.62}{0.19} \\
 & Screen photo & \meanstd{92.62}{0.29} & \meanstd{\mathbf{92.67}}{0.64} & \meanstd{92.45}{0.14} & \meanstd{92.20}{0.50} & \meanstd{92.19}{0.27} & \meanstd{92.52}{0.56} & \meanstd{92.28}{0.12} \\
\cmidrule(lr){1-9}
\multirow{2}{*}{\textbf{DocPTBench}} & SynReal & \meanstd{73.79}{0.39} & \meanstd{\mathbf{74.49}}{1.23} & \meanstd{74.10}{0.55} & \meanstd{73.81}{1.14} & \meanstd{74.21}{1.17} & \meanstd{73.33}{0.60} & \meanstd{73.60}{0.45} \\
 & Dewarp & \meanstd{83.76}{0.51} & \meanstd{83.92}{0.61} & \meanstd{83.69}{0.13} & \meanstd{83.36}{0.50} & \meanstd{83.94}{1.09} & \meanstd{83.83}{0.48} & \meanstd{\mathbf{83.99}}{0.42} \\
\cmidrule(lr){1-9}
\textbf{Wild-OmniDocBench} & --- & \meanstd{\mathbf{85.02}}{0.21} & \meanstd{84.79}{0.57} & \meanstd{84.81}{0.75} & \meanstd{84.55}{0.56} & \meanstd{84.74}{0.37} & \meanstd{84.28}{0.38} & \meanstd{84.68}{0.54} \\
\bottomrule
\end{tabular}%
}
\par\vspace{3pt}
{\scriptsize\raggedright
\textit{Note.} Bold marks the highest mean per split. w/o UAP: \blockmethod{} + SCR; w/o SCR: \blockmethod{} + UAP. --- denotes an aggregate benchmark score.\par}
\end{table}

\begin{table}[h]
\centering
\caption{Component diagnostics on OmniDocBench v1.6 after 100 RL updates from the shared 700K-page SFT initialization. Parsing metrics report five-seed means $\pm$ standard deviations; best and second-best values are bold and underlined, respectively.}
\label{tab:component-diagnostics}
\resizebox{\linewidth}{!}{%
\begin{tabular}{@{}lccccccc@{}}
\toprule
& \multicolumn{5}{c}{Parsing metrics} & \multicolumn{2}{c}{Omission rates (\%)} \\
\cmidrule(lr){2-6}\cmidrule(l){7-8}
Method & Text ED$\downarrow$ & Formula CDM$\uparrow$ & Formula ED$\downarrow$ & Table TEDS$\uparrow$ & Order ED$\downarrow$ & Element$\downarrow$ & Local$\downarrow$ \\
\midrule
\blockmethod{} & \meanstd{0.0360}{0.0018} & \meanstd{0.9340}{0.0028} & \meanstd{0.1279}{0.0020} & \meanstd{0.8765}{0.0046} & \meanstd{0.0475}{0.0025} & $1.09$ & $6.72$ \\
\blockmethod{} + UAP & \meanstd{\mathbf{0.0333}}{0.0017} & \meanstd{\underline{0.9398}}{0.0042} & \meanstd{\underline{0.1261}}{0.0041} & \meanstd{0.8786}{0.0059} & \meanstd{\underline{0.0459}}{0.0020} & $\underline{1.01}$ & $\underline{5.59}$ \\
\blockmethod{} + SCR & \meanstd{0.0364}{0.0008} & \meanstd{0.9325}{0.0046} & \meanstd{0.1275}{0.0031} & \meanstd{\underline{0.8807}}{0.0080} & \meanstd{0.0500}{0.0017} & $1.08$ & $5.99$ \\
\textbf{P2TPO} & \meanstd{\underline{0.0343}}{0.0015} & \meanstd{\mathbf{0.9446}}{0.0027} & \meanstd{\mathbf{0.1212}}{0.0023} & \meanstd{\mathbf{0.8833}}{0.0061} & \meanstd{\mathbf{0.0423}}{0.0024} & $\mathbf{0.95}$ & $\mathbf{5.53}$ \\
\bottomrule
\end{tabular}%
}
\par\vspace{3pt}
{\scriptsize\raggedright
\textit{Note.} ED: edit distance. Parsing metrics use $[0,1]$; omission rates use $[0,100]$. Element omission counts missing eligible reference elements; local omission counts omitted characters within matched text blocks and excludes whole-element omissions. Omission-rate standard deviations are unavailable. Appendix~\ref{app:diagnostics} gives the complete definitions.\par}
\end{table}

\paragraph{Parsing accuracy and omissions.}
On OmniDocBench, adding UAP to \blockmethod{} reduces the observed element-omission rate from 1.09\% to 1.01\% and local omission from 6.72\% to 5.59\% (Table~\ref{tab:component-diagnostics}). Full P2TPO records the lowest observed omission rates and the best formula, table, and reading-order metrics among the four component configurations; \blockmethod{} + UAP gives the lowest Text ED. SCR's Formula CDM and reading-order gains occur with UAP enabled, matching the conditional pattern in the Overall-score ablation. Appendix~\ref{app:diagnostics} specifies the omission denominators and reporting scope.

Figure~\ref{fig:credit-sign-reversal} illustrates the routing resolution of reference-element feedback using two frozen eight-response groups (no policy update; Appendix~\ref{app:qualitative-records}). On the common sampled responses, \blockmethod{} produces block-dependent magnitudes and mixed-sign coefficients, while GRPO uses a uniform coefficient within each response.

\begin{figure}
\centering
\begin{minipage}[t]{0.30\linewidth}
\vspace{0pt}
\centering
\includegraphics[width=\linewidth]{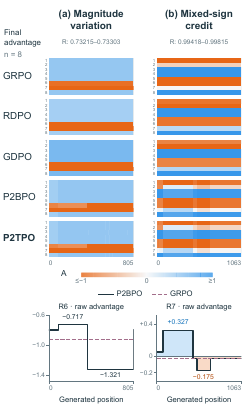}
\caption{Frozen-response credit assignment ($n=8$ per case; no policy update). Heatmaps saturate beyond $\pm1$. Lower profiles compare GRPO and \blockmethod{} for (a) R6 and (b) R7; y-scales differ.}
\label{fig:credit-sign-reversal}
\end{minipage}\hfill
\begin{minipage}[t]{0.68\linewidth}
\vspace{0pt}
\centering
\includegraphics[width=\linewidth]{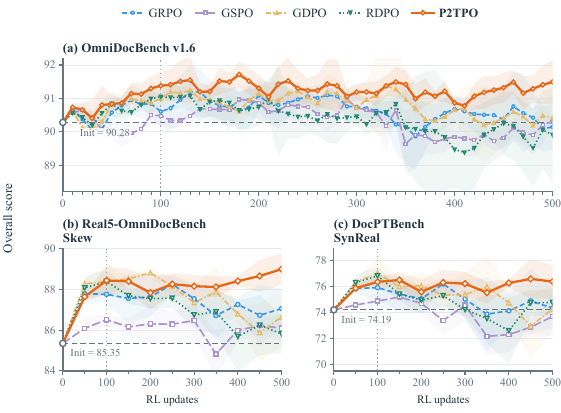}
\caption{Extended Qwen3.5-2B/50K training with 64,000 RL records. Lines show unsmoothed five-run means; bands denote $\pm1$ SD. Markers and ticks indicate evaluations every ten updates in (a) and every 50 in (b,c). Each panel starts at its measured Init (dashed line), with SD bands beginning at the first evaluation. Panel-specific y-axis ranges truncate some SD bands; vertical guides mark update 100.}
\label{fig:extended-training}
\end{minipage}
\end{figure}

\subsection{Generality across initialization, scale, and backbone}
\label{sec:generality}

Table~\ref{tab:initialization-backbone} complements the primary Qwen3.5-2B/700K results with three additional SFT initialization settings: Qwen3.5-2B/50K, Qwen3.5-4B/50K, and Qwen3-VL-2B-Instruct/700K. Within each setting, all RL methods share the starting checkpoint and 100-update budget, with five runs using distinct RL data seeds. The table reports means on the same six benchmarks and 13 splits as Table~\ref{tab:all-benchmarks}; Appendix~\ref{app:generality-efficiency} provides standard deviations and model-setting details.

\begin{table}[h]
\centering
\caption{Cross-benchmark RL mean scores ($\uparrow$) across three SFT-initialized checkpoints. Benchmark names and splits follow Table~\ref{tab:all-benchmarks}.}
\label{tab:initialization-backbone}
\scriptsize
\setlength{\tabcolsep}{0.85pt}
\renewcommand{\arraystretch}{1.15}
\resizebox{\linewidth}{!}{%
\begin{tabular}{@{}l@{\hspace{3pt}}l@{\hspace{3pt}}*{5}{c}@{\hspace{4pt}}*{5}{c}@{\hspace{4pt}}*{5}{c}@{}}
\toprule
\multicolumn{2}{@{}l}{\textbf{Initial Conditions}} & \multicolumn{5}{c}{\shortstack{\textbf{Qwen3.5-2B}\\SFT data: 50K pages}} & \multicolumn{5}{c}{\shortstack{\textbf{Qwen3.5-4B}\\SFT data: 50K pages}} & \multicolumn{5}{c@{}}{\shortstack{\textbf{Qwen3-VL-2B-Instruct}\\SFT data: 700K pages}} \\
\cmidrule(lr){3-7}\cmidrule(lr){8-12}\cmidrule(lr){13-17}
Benchmark & Split & GRPO & GSPO & GDPO & RDPO & \textbf{P2TPO} & GRPO & GSPO & GDPO & RDPO & \textbf{P2TPO} & GRPO & GSPO & GDPO & RDPO & \textbf{P2TPO} \\
\midrule
OmniDocBench & v1.6 & $91.00$ & $90.39$ & $\underline{91.05}$ & $90.85$ & $\mathbf{91.62}$ & $\underline{92.13}$ & $91.99$ & $92.07$ & $91.54$ & $\mathbf{92.53}$ & $90.45$ & $90.55$ & $\underline{90.57}$ & $88.25$ & $\mathbf{91.29}$ \\
\addlinespace[2pt]
\multirow{3}{*}{PureDocBench} & Clean & $80.83$ & $81.01$ & $\underline{81.02}$ & $79.75$ & $\mathbf{81.98}$ & $\underline{84.41}$ & $84.28$ & $84.07$ & $82.90$ & $\mathbf{84.95}$ & $\underline{81.81}$ & $81.80$ & $81.75$ & $79.67$ & $\mathbf{81.95}$ \\
 & Digital & $\underline{78.26}$ & $78.11$ & $77.85$ & $77.27$ & $\mathbf{79.16}$ & $\underline{81.80}$ & $81.06$ & $81.47$ & $80.14$ & $\mathbf{82.34}$ & $76.08$ & $76.29$ & $\mathbf{76.57}$ & $73.98$ & $\underline{76.51}$ \\
 & Real & $\underline{69.82}$ & $68.66$ & $69.58$ & $69.20$ & $\mathbf{71.01}$ & $73.55$ & $72.52$ & $\underline{73.61}$ & $71.86$ & $\mathbf{74.48}$ & $65.45$ & $65.44$ & $\underline{65.76}$ & $63.53$ & $\mathbf{66.25}$ \\
\addlinespace[2pt]
FinixDocBench & --- & $69.45$ & $70.11$ & $\underline{70.95}$ & $\underline{70.95}$ & $\mathbf{73.01}$ & $74.33$ & $\underline{74.43}$ & $73.59$ & $69.38$ & $\mathbf{76.74}$ & $67.68$ & $67.75$ & $\underline{67.95}$ & $65.18$ & $\mathbf{68.23}$ \\
\addlinespace[2pt]
\multirow{5}{*}{Real5-OmniDocBench} & Scanning & $90.66$ & $90.23$ & $\underline{90.97}$ & $\underline{90.97}$ & $\mathbf{91.29}$ & $\underline{92.04}$ & $91.89$ & $91.81$ & $91.52$ & $\mathbf{92.23}$ & $\underline{90.07}$ & $89.42$ & $89.89$ & $87.79$ & $\mathbf{90.46}$ \\
 & Warping & $\underline{89.56}$ & $89.06$ & $89.47$ & $89.53$ & $\mathbf{89.86}$ & $91.02$ & $90.84$ & $\underline{91.18}$ & $90.22$ & $\mathbf{91.45}$ & $85.19$ & $84.17$ & $\underline{85.62}$ & $84.46$ & $\mathbf{86.37}$ \\
 & Skew & $88.02$ & $87.04$ & $\underline{88.11}$ & $87.61$ & $\mathbf{88.44}$ & $\underline{90.95}$ & $90.40$ & $90.78$ & $90.02$ & $\mathbf{91.33}$ & $76.83$ & $77.74$ & $\underline{77.82}$ & $76.66$ & $\mathbf{79.79}$ \\
 & Illumination & $90.81$ & $90.35$ & $\underline{91.10}$ & $90.83$ & $\mathbf{91.23}$ & $92.48$ & $92.33$ & $\underline{92.59}$ & $91.71$ & $\mathbf{92.80}$ & $89.26$ & $88.79$ & $\underline{89.38}$ & $87.15$ & $\mathbf{89.65}$ \\
 & Screen photo & $89.79$ & $89.36$ & $\underline{90.43}$ & $90.24$ & $\mathbf{90.50}$ & $\underline{91.91}$ & $91.71$ & $91.87$ & $91.08$ & $\mathbf{92.16}$ & $\underline{88.46}$ & $88.09$ & $88.39$ & $86.01$ & $\mathbf{89.33}$ \\
\addlinespace[2pt]
\multirow{2}{*}{DocPTBench} & SynReal & $75.57$ & $75.21$ & $\mathbf{76.47}$ & $76.23$ & $\underline{76.31}$ & $\underline{77.86}$ & $76.51$ & $77.76$ & $76.37$ & $\mathbf{78.19}$ & $\underline{67.80}$ & $67.18$ & $67.23$ & $65.91$ & $\mathbf{68.94}$ \\
 & Dewarp & $82.43$ & $82.47$ & $\underline{83.26}$ & $82.90$ & $\mathbf{83.63}$ & $84.02$ & $83.98$ & $\underline{84.11}$ & $83.23$ & $\mathbf{84.88}$ & $\underline{78.28}$ & $77.38$ & $77.97$ & $76.16$ & $\mathbf{79.34}$ \\
\addlinespace[2pt]
Wild-OmniDocBench & --- & $83.67$ & $82.99$ & $83.62$ & $\underline{84.27}$ & $\mathbf{84.48}$ & $\underline{86.04}$ & $85.34$ & $85.83$ & $85.10$ & $\mathbf{86.28}$ & $79.20$ & $78.58$ & $\underline{79.65}$ & $77.78$ & $\mathbf{80.32}$ \\
\bottomrule
\end{tabular}
}
\par\vspace{3pt}
\textit{Note.} {\scriptsize\raggedright Best and second-best RL means within each initial condition and split are bold and underlined, respectively. All scores are measured after 100 RL updates. Pre-RL scores (Init) and full mean $\pm$ standard deviation results appear in Table~\ref{tab:initialization-backbone-full}. --- denotes an aggregate benchmark score.\par}
\end{table}

In the three additional model and initialization settings evaluated, P2TPO achieves the highest mean in 37 out of 39 split-setting comparisons. This includes all 13 comparisons with Qwen3.5-4B/50K, and P2TPO also outperforms GSPO in all 39 cases. These results demonstrate its effectiveness across the tested initial policies, model scales, and backbones. Full details, including means, standard deviations, and split-specific comparisons, can be found in the Appendix~\ref{app:generality-efficiency}.

\subsection{Hyperparameter sensitivity and extended training}
\label{sec:robustness}

\paragraph{Hyperparameter sensitivity.}
Across one-at-a-time sweeps, all configurations stay within 0.95 points of the default on each split and exceed the strongest response-level baseline on at least 11 of 13 splits, indicating robustness within the tested ranges (Table~\ref{tab:parameter-sensitivity}, Appendix~\ref{app:parameter-sensitivity}).

\paragraph{Extended training horizon.}
The 500-update study uses Qwen3.5-2B/50K, 64,000 RL records, and five runs per method, with OmniDocBench v1.6 evaluated every ten updates and the other 12 splits every 50. At update 500, P2TPO attains the highest mean on all 13 splits. On Real5-OmniDocBench Skew and DocPTBench SynReal, its final margins over the strongest response-level baseline are 1.94 and 1.60 points, respectively (Figure~\ref{fig:extended-training}). These panels illustrate improvement and retention relative to this series' update-100 checkpoints. Appendix~\ref{app:extended-training} provides the separate-run protocol, trajectories, and all additional checkpoints.

\section{Theoretical Analysis}
\label{sec:structured-feedback-theory}

\paragraph{What is lost in aggregation?}
\blockmethod{} compares individual reference blocks; element-group feedback pools multiple blocks, such as those of the same type, and page feedback pools the groups. For common matched rollouts, let $\mathbf Q$ contain scores of $K$ blocks, $\mathbf H$ center across rollouts, and $\mathbf B=\mathbf H\mathbf Q$. With fixed positive normalized weights, $\mathbf M$ pools blocks into $C$ nonempty groups and $\boldsymbol\alpha$ combines their scores into an additive page utility:
\begin{equation}
\underbrace{\mathbf B}_{\text{reference blocks}}
\;\longmapsto\;
\underbrace{\mathbf B\mathbf M}_{\text{element groups}}
\;\longmapsto\;
\underbrace{\mathbf B\mathbf M\boldsymbol\alpha}_{\text{page}}.
\label{eq:main-contrast-hierarchy}
\end{equation}
When $1<C<K$, both maps are non-invertible: grouping can cancel opposing block contrasts, while page aggregation can additionally cancel contrasts between groups. Retaining reference-block identity preserves distinctions that these aggregations can erase before normalization. Appendix~\ref{subsec:block-contrasts} proves this result (Proposition~\ref{prop:block-information}) and quantifies the coefficient-reconstruction error imposed by coarser routing (Lemma~\ref{lem:routing-resolution}).

Appendix~\ref{subsec:uap-theory} establishes how UAP activates and routes feedback for omitted or uniformly imperfect outcomes. Appendix~\ref{subsec:span-scr} proves SCR's source-block budget conservation and the conditions on redistribution strength and mask precision under which more credit reaches the intended targets.

\section{Discussion and Limitations}
\label{sec:limitations}

P2TPO constructs element-level feedback from ground-truth references, so its alignment and quality scores depend on annotation accuracy and serialization conventions. Annotation noise, repeated text, and alternative valid serializations can introduce ambiguity into element matching and token labels. Reliable references and alignment procedures are therefore important conditions for applying the framework.
Our experiments focus on text, formulas, and tables; code and figures are not evaluated as distinct element types. Extending the framework to these elements requires suitable representations, alignment procedures, and quality measures. Future work could investigate their joint optimization within an end-to-end reinforcement learning framework.
Finally, SCR currently allocates credit between aligned correct and erroneous token regions, with a shared redistribution weight within each region. Quality gating and cap-and-rebate redistribution bound concentration while preserving the source block's credit budget. Distinguishing individual token contributions within these regions is a direction for finer localization, with its effects on optimization to be evaluated.

\begin{samepage}
\section{Conclusion}

We introduced P2TPO for structure-aligned RL post-training of full-page document parsers. It preserves reference-element identity, completes feedback for missing outcomes and imperfect consensus through UAP, and localizes matched-block credit through budget-preserving SCR. Our analysis characterizes local quality contrasts that fixed additive aggregation can erase. From a shared supervised checkpoint, P2TPO exceeds four response-level baselines in five-seed mean score on all 13 primary splits across six benchmarks, including a 0.74-point gain over the strongest baseline on OmniDocBench v1.6. Additional settings and ablations support its applicability and the complementarity of outcome feedback and alignment-informed token allocation.
\par
\end{samepage}

\section*{AI Use Statement}
During the preparation of this manuscript, we utilized generative artificial intelligence tools to assist with literature search and discovery, the drafting of certain sections, linguistic refinement and optimization of expression, the review of some figures and tables, and the generation of illustrative material, as well as the construction of proofs for mathematical propositions. We did not use generative artificial intelligence to formulate research questions, develop core research ideas or design research methods. The authors reviewed and verified all content generated or revised with the assistance of artificial intelligence, including confirming that the revised wording retains the original meaning, cross-checking cited references against their sources, and verifying the validity of mathematical arguments and the assumptions on which they rely. The authors assume full responsibility for the final content of this study, including the text, references, mathematical arguments, illustrative material, and the reported results and conclusions.
\bibliography{iclr2027_conference}
\bibliographystyle{iclr2027_conference}

\clearpage
\appendix
\section*{Appendix overview}

The supplementary material presents reward-interface comparisons, method details and algorithms, experimental protocols and results, theoretical analysis, and qualitative case studies, in that order. Section titles and page numbers below link directly to their destinations.

\begingroup
\newcommand{\appsectionentry}[1]{\addlinespace[5pt]\multicolumn{2}{@{}l}{\textbf{\hyperref[#1]{\ref*{#1}\quad\nameref*{#1}}}} & \hyperref[#1]{\pageref*{#1}}\\}
\newcommand{\appsubentry}[1]{\hspace{0.7em}\hyperref[#1]{\ref*{#1}} & \hyperref[#1]{\nameref*{#1}} & \hyperref[#1]{\pageref*{#1}}\\}
\begin{tabularx}{\linewidth}{@{}lXr@{}}
\toprule
\multicolumn{2}{@{}l}{\textbf{Contents}} & \textbf{Page}\\
\midrule
\appsectionentry{app:supplementary-related-work}
\appsubentry{app:reward-interfaces}
\appsubentry{app:related-methods}
\appsubentry{app:parsing-recipe-scope}
\appsectionentry{app:method-details}
\appsubentry{app:credit-specification}
\appsubentry{app:credit-pseudocode}
\appsectionentry{app:experimental-details}
\appsubentry{app:experimental-protocols}
\appsubentry{app:mechanism-study}
\appsubentry{app:qualitative-records}
\appsubentry{app:generality-efficiency}
\appsubentry{app:parameter-sensitivity}
\appsubentry{app:extended-training}
\appsubentry{app:runtime}
\appsectionentry{app:theory-proofs}
\appsubentry{subsec:block-contrasts}
\appsubentry{subsec:uap-theory}
\appsubentry{subsec:span-scr}
\appsectionentry{app:case-studies}
\bottomrule
\end{tabularx}
\endgroup

\clearpage
\section{Supplementary Related Work and Reward Interfaces}
\label{app:supplementary-related-work}

This section compares reward targets and optimizer signals, discusses CSPO-inspired table scoring as a possible extension of reference-instance feedback, and clarifies the scope of the reported parsing recipes. The evaluated baselines and component interventions are specified in Appendix~\ref{app:comparison-methods}; the proposed scoring extension is not part of the reported experiments.

\subsection{Reward targets and optimizer signals}
\label{app:reward-interfaces}

Table~\ref{tab:reward-routing} compares the prediction targets, reward construction, and optimizer signals of the reported document-parsing RL recipes. Each target label denotes the task used by that recipe, and each signal label denotes how its rewards reach generated tokens. GRPO-based optimization can accommodate different reward interfaces; the table classifies the interfaces used in the cited recipes.

\begin{table}[h]
\centering
\caption{RL targets and optimizer signals in document parsing. The signal column distinguishes response aggregation, functional-attribute routing, and reference-block credit assignment. Target and signal labels are defined below.}
\label{tab:reward-routing}
\scriptsize
\setlength{\tabcolsep}{3pt}
\renewcommand{\arraystretch}{1.12}
\begin{tabularx}{\linewidth}{@{}>{\raggedright\arraybackslash}p{0.19\linewidth}>{\raggedright\arraybackslash}p{0.12\linewidth}>{\raggedright\arraybackslash}p{0.10\linewidth}>{\raggedright\arraybackslash}X>{\raggedright\arraybackslash}p{0.20\linewidth}@{}}
\toprule
Work & Target & RL algorithm & Reward construction & Optimizer signal \\
\midrule
Infinity-Parser~\citep{wang2025infinityparser} & end2end & GRPO & Edit similarity $+$ paragraph count $+$ reading order & Sum $\to$ Scalar \\
Logics-Parsing~\citep{chen2025logicsparsing} & end2end & GRPO & Text similarity $+$ box localization $+$ reading order & Sum $\to$ Scalar \\
olmOCR~2~\citep{poznanski2025olmocr2} & end2end & GRPO & Unit-test pass rate $+$ termination and metadata checks & Sum $\to$ Scalar \\
HunyuanOCR~\citep{hunyuan2025ocr} & end2end & GRPO & Parsing: NED-based similarity; format/length validity gate & Scalar \\
LightOnOCR\newline\citep{taghadouini2026lightonocr} & end2end & GRPO & OCR: unit-test pass rate with repetition, termination, math, and format checks & Scalar \\
TexOCR~\citep{wang2026texocr} & end2end & GRPO-style & Pass rate of binary tests derived from nine metrics & Scalar \\
Infinity-Parser2~\citep{huang2026infinityparser2} & end2end; table; formula & GRPO & Task-selected EDS/TEDS/CDM; doc2json sums spatial mIoU and text EDS & Scalar;\newline Sum $\to$ Scalar (doc2json) \\
\midrule
MonkeyOCR v1.5~\citep{zhang2025monkeyocr15} & table & GRPO & Rule check and learned render-and-compare VLM reward & Scalar \\
StrucTab~\citep{li2026structab} & table & GRPO & Weighted validity $+$ prefix structure $+$ destylized content & Sum $\to$ Scalar \\
CSPO~\citep{yang2026cspo} & table & GRPO-based & Table-wide functional-attribute rewards; global TEDS $+$ compilation & Attribute $\to$ Token\newline (+ global) \\
DocPO~\citep{wang2026docpo} & text; formula; table & GRPO-style & Type-specific patch rewards with training-step sharpening & Scalar \\
PaddleOCR-VL-1.6~\citep{zhang2026paddleocrvl16} & element & GRPO & Task-specific $\mathrm{Valid}\!\times\!\mathrm{Struct}\!\times\!\mathrm{Similarity}$ & Product $\to$ Scalar \\
\midrule
\textbf{P2TPO (ours)} & end2end & P2TPO & Type-specific scores for individual matched/missing reference blocks; global page reward & Element type $\to$\newline Ref. block $\to$ Token\newline (UAP/SCR; + global) \\
\bottomrule
\end{tabularx}
\par\vspace{3pt}
{\scriptsize\raggedright
\textit{Target.} \textbf{end2end}: complete heterogeneous page transcription; \textbf{table}, \textbf{formula}, \textbf{text}: recognition of the specified content type. \textbf{element} is reserved for a shared, type-prompted region recognizer within a staged page parser; each call transcribes only that region. It does not denote P2TPO's internal reference elements. A single-type RL recipe retains its specific target label, e.g., MonkeyOCR's table recipe.\par
\smallskip
\textit{Optimizer signal.} \textbf{Scalar}: one normalized response advantage broadcast over its tokens; 
a test pass rate still belongs here. \textbf{Sum/Product $\to$ Scalar}: aggregate reward components before normalization. \textbf{Attribute $\to$ Token}: CSPO normalizes one reward per functional attribute class across rollouts and routes its advantage to that class's code tokens. \textbf{Element type $\to$ Ref. block $\to$ Token}: type selects the scoring rule; each reference-block identity defines its own comparison group; UAP completes missing outcomes and routes their credit; SCR redistributes matched-block credit within its span. Both token-routing methods retain global feedback. Infinity-Parser2 selects a task-specific reward function.\par}
\end{table}

Here ANLS is average normalized Levenshtein similarity, RoI is region of interest, IoU is intersection over union, mIoU is mean IoU, NED is normalized edit distance, EDS is edit-distance similarity, TEDS is tree-edit-distance similarity, and CDM is character detection matching; VLM denotes a vision-language model.

\subsection{Complementary mechanisms and possible adaptations}
\label{app:related-methods}
\paragraph{Attribute-informed block scoring.}
A table-block scorer could use CSPO-inspired functional attributes~\citep{yang2026cspo} to assess content and structure, then supply the resulting block scores to P2TPO's reference-instance comparisons. This would refine evaluation within aligned table blocks while retaining UAP and SCR.

\subsection{Scope of reported parsing recipes}
\label{app:parsing-recipe-scope}

\paragraph{Recipe scope and auxiliary tasks.}
HunyuanOCR's row describes its document-parsing RL recipe, not every task supported by the model. LightOnOCR's row covers its OCR recipe; its separate bounding-box recipe uses IoU-based localization rewards and is not element transcription. PaddleOCR-VL-1.6 uses layout detection followed by task-prompted region recognition; its RL tasks include text, formulas, and tables as well as charts, seals, and spotting. DocPO's row covers mixed text/formula/table patch recognition. Its step-aware sharpening follows the training step and still yields one response coefficient.

\paragraph{Document question answering.}
DocR1~\citep{xiong2025docr1} and DocThinker~\citep{yu2025docthinker} optimize answers and reasoning rather than page transcriptions, so they are outside the table's parsing targets. DocR1's EviGRPO sums format, answer ANLS, and evidence-page F1 rewards. DocThinker combines format, answer accuracy, RoI IoU, and question-rephrasing rewards before GRPO. Both use a Sum $\to$ Scalar interface.

\section{Credit Construction and Policy Update Algorithms}
\label{app:method-details}

This appendix specifies the procedures underlying Section~\ref{sec:method}: coefficient construction (Algorithm~\ref{alg:p2tpo-credit}), SCR redistribution (Algorithm~\ref{alg:p2tpo-scr}), and the policy update (Algorithm~\ref{alg:p2tpo-update}). Appendix~\ref{app:training-config} provides the training settings.

\subsection{Algorithmic inputs and credit flow}
\label{app:reward-construction}
\label{app:alignment-details}
\label{app:credit-specification}

For each page's rollout group, the evaluation interface supplies raw page-reward dimensions $z_{i,d}$ and availability masks $v_{i,d}$, reference-indexed scores $s_{i,k}\in[0,1]$, matched and missing index sets $\mathcal K_i$ and $\mathcal M_i$, and generated-token supports $\mathcal B_{i,k}$. SCR additionally receives disjoint correct- and error-span masks for eligible blocks. Missing-element contributions use the normalized routing weights $\nu_{i,k}(t)$ defined in Section~\ref{sec:objective}.

\paragraph{Page-reward construction.}
Let $v_{i,d}\in\{0,1\}$ indicate whether raw reward dimension $z_{i,d}$ is available for rollout $i$. For rollouts with at least one available dimension, define
\[
r_i=\frac{\sum_d v_{i,d}\operatorname{clip}(z_{i,d},0,1)}{\sum_d v_{i,d}}.
\]
Let $\mathcal G$ contain the valid, nonempty rollouts for the same prompt with at least one available reward dimension. For $|\mathcal G|\ge2$, let $\mu_{\mathcal G}$ and $\sigma_{\mathcal G}$ be the mean and population standard deviation of $(r_i)_{i\in\mathcal G}$. We set $A_i^{\mathrm P}=(r_i-\mu_{\mathcal G})/\sigma_{\mathcal G}$ when $\sigma_{\mathcal G}>10^{-8}$ and $A_i^{\mathrm P}=r_i-\mu_{\mathcal G}$ otherwise. Excluded rollouts and groups with fewer than two eligible rollouts receive zero page advantage.

\paragraph{SCR mask construction.}
For type-matched text, formula, and table blocks, we compare predicted token pieces with pieces from the aligned reference using \texttt{SequenceMatcher} with \texttt{autojunk=False}. Comparison keys strip leading and trailing whitespace from each piece. Equal pieces receive label $+1$, and inserted or replaced predicted pieces receive $-1$. A reference deletion assigns $-1$ to the next predicted token, or the final token at the block boundary, overriding any earlier match label. For nonempty reference-piece sequences, whitespace-only pieces and standalone display-math delimiter pieces are then assigned $0$.

Error labels seed the error mask. We merge error runs separated by at most $g$ non-error tokens, including the intervening tokens in the merged mask. Positive labels outside this mask form the correct mask. SCR selects the error mask for a negative block coefficient and the correct mask for a positive coefficient. Redistribution requires the raw block reward to be at least its type-specific quality threshold. Algorithm~\ref{alg:p2tpo-scr} converts the selected mask into nonnegative weights with unit mean over the block.

\paragraph{Token mapping.}
Predicted pieces use response-text token offsets; reference pieces use tokenizer offsets without special tokens, with per-token decoding as a fallback. Reference formula delimiters are removed before piece construction. Packed token endpoints determine each block's span.

\subsection{Credit construction and policy update}
\label{app:credit-pseudocode}

\subsubsection{Page-to-token coefficients}
All comparisons are performed within a page's rollout group. Binary switches $U$ and $S$ enable UAP and SCR: $(U,S)=(0,0)$ gives \blockmethod{}, $(1,0)$ adds UAP, $(0,1)$ adds SCR, and $(1,1)$ gives full P2TPO. With UAP, the comparison set $\mathcal S_k$ contains all valid rollouts; otherwise it contains only rollouts matching reference $k$. Setting $\lambda=0$ removes consensus feedback while retaining outcome completion and routing; setting $\beta=0$ gives uniform block allocation.

\begin{algorithm}
\caption{Construct page-to-token coefficients}
\label{alg:p2tpo-credit}
\begin{algorithmic}[1]
\Require Raw page-reward dimensions and availability masks, reference outcomes, token supports, masks, and routing weights
\Require Switches $U,S$; weights $\omega,\lambda$; dispersion tolerance $\epsilon$
\Ensure Token advantages $A$ for the sampled responses
\For{each page's rollout group}
    \State Clip raw page-reward dimensions to $[0,1]$ and compute masked means $r_i$ (Appendix~\ref{app:credit-specification})
    \State Normalize $(r_i)_i$ within the eligible same-prompt group to obtain $A_i^{\mathrm P}$
    \State Initialize local contributions $C_{i,t}\gets0$
    \For{each reference $k$ with nonempty comparison set $\mathcal S_k$}
        \State $Q_{i,k}\gets s_{i,k}$ if matched, otherwise $0$, for $i\in\mathcal S_k$
        \State $(\mu_k,\sigma_k)\gets$ mean and population standard deviation of $(Q_{i,k})_{i\in\mathcal S_k}$
        \If{$\sigma_k>\epsilon$}
            \State $a_{i,k}\gets(Q_{i,k}-\mu_k)/\sigma_k$ for $i\in\mathcal S_k$
        \Else
            \State $a_{i,k}\gets-U\lambda\max(0,1-\mu_k)$ for $i\in\mathcal S_k$
        \EndIf
        \For{each $i\in\mathcal S_k$}
            \If{$k\in\mathcal K_i$}
                \State $d\gets\Call{SCR}{(i,k),a_{i,k},S}$ \Comment{Algorithm~\ref{alg:p2tpo-scr}}
                \State $C_{i,t}\gets C_{i,t}+a_{i,k}d_t$ for $t\in\mathcal B_{i,k}$
            \Else
                \State $C_{i,t}\gets C_{i,t}+a_{i,k}\nu_{i,k}(t)$ for response tokens $t$
            \EndIf
        \EndFor
    \EndFor
    \State $A_{i,t}\gets(1-\omega)A_i^{\mathrm P}+\omega C_{i,t}$ for response tokens $t$
\EndFor
\State \Return $A$, held fixed during policy optimization
\end{algorithmic}
\end{algorithm}

\subsubsection{Span credit redistribution}
\label{app:scr-procedure}
Algorithm~\ref{alg:p2tpo-scr} selects the error-span mask for negative coefficients and the correct-span mask for positive coefficients. It mixes uniform and selected-token allocation with strength $\beta\in[0,1]$, applies a pre-rebate cap $D>1$, and uniformly restores the removed coefficient mass. For a nonempty selected fraction $p_{i,k}=N_{i,k}/L_{i,k}$, this is equivalent to Equation~\ref{eq:scr} with
\begin{equation}
\eta_{i,k}=\min\{\beta,p_{i,k}(D-1+\beta)\}.
\label{eq:scr-effective-strength}
\end{equation}
The resulting weights have unit mean and preserve the block's signed and absolute credit budgets. Proposition~\ref{prop:scr-concentration} in Appendix~\ref{subsec:span-scr} establishes this equivalence and the conservation identities.

\begin{algorithm}
\caption{Redistribute matched-block credit with SCR}
\label{alg:p2tpo-scr}
\begin{algorithmic}[1]
\Require Block $(i,k)$ of length $L$, coefficient $a$, and switch $S$
\Require Eligibility and correct/error masks (Appendix~\ref{app:credit-specification}), strength $\beta$, and pre-rebate cap $D$
\Ensure Nonnegative weights $d$ with $\sum_t d_t=L$
\Function{SCR}{$(i,k),a,S$}
    \If{$S=0$, $a=0$, or the block is ineligible} \State \Return $\mathbf1_L$ \EndIf
    \State Select error mask $u$ if $a<0$, otherwise select correct mask $u$
    \State $N\gets\sum_{t=1}^{L}u_t$
    \If{$N=0$} \State \Return $\mathbf1_L$ \EndIf
    \State $d^0_t\gets1-\beta+\beta L u_t/N$ for each $t$
    \State $v_t\gets\min(d^0_t,D)$ for each $t$
    \State $d_t\gets v_t+1-L^{-1}\sum_{j=1}^{L}v_j$ for each $t$
    \State \Return $d$
\EndFunction
\end{algorithmic}
\end{algorithm}

\subsubsection{Policy objective and update}
The policy update uses the clipped PPO surrogate~\citep{schulman2017ppo}.
For sampled tokens, the importance ratio is $\rho_{i,t}(\theta)=\pi_\theta(y_{i,t}\mid x,p,y_{i,<t})/\pi_{\theta_{\mathrm{old}}}(y_{i,t}\mid x,p,y_{i,<t})$. With clipping width $\varepsilon$ and dual-clip factor $\gamma$, define
\begin{equation}
\begin{aligned}
f(\rho,A)&=\min\{\rho A,\operatorname{clip}(\rho,1-\varepsilon,1+\varepsilon)A\},\\
f_{\mathrm{dual}}(\rho,A)&=\begin{cases}
\max\{f(\rho,A),\gamma A\},&A<0,\\
f(\rho,A),&A\ge0.
\end{cases}
\end{aligned}
\end{equation}
For minibatch $\mathcal D$ and response-token mask $m_{i,t}$, the objective averages over active response tokens:
\begin{equation}
\widehat{\mathcal J}(\theta)=\frac{1}{N_{\mathrm{act}}}
\sum_{i\in\mathcal D}\sum_{t=1}^{T_i}m_{i,t}
 f_{\mathrm{dual}}\!\left(\rho_{i,t}(\theta),A_{i,t}^{\mathrm{US}}\right),
\qquad N_{\mathrm{act}}=\sum_{i\in\mathcal D}\sum_{t=1}^{T_i}m_{i,t}.
\label{eq:objective}
\end{equation}
Advantages and old-policy probabilities remain fixed while differentiating this objective. Component variants use the same loss with their corresponding token coefficients; Table~\ref{tab:training-config} gives the optimization settings.

\begin{algorithm}
\caption{Optimize the shared clipped policy objective}
\label{alg:p2tpo-update}
\begin{algorithmic}[1]
\Require Sampled responses, old-policy probabilities, token advantages $A$, and masks $m$
\Require Current policy $\pi_\theta$, clipping width $\varepsilon$, and dual-clip factor $\gamma$
\Ensure Updated policy parameters $\theta$
\For{each optimizer minibatch $\mathcal D$}
    \State Evaluate current-policy probabilities and form importance ratios $\rho_{i,t}(\theta)$
    \State Compute $f_{\mathrm{dual}}(\rho_{i,t},A_{i,t})$ with fixed advantages and old probabilities
    \State $\mathcal L\gets-\widehat{\mathcal J}(\theta)$ using the response-token mean in Equation~\ref{eq:objective}
    \State Update $\theta$ by an optimizer step using $\nabla_\theta\mathcal L$
\EndFor
\State \Return $\theta$
\end{algorithmic}
\end{algorithm}

\section{Experimental Details and Supplementary Results}
\label{app:experimental-details}
\label{app:additional-analyses}

Appendix~\ref{app:experimental-protocols} collects the shared data, training, comparison, and evaluation protocols. The supplementary results then follow Section~\ref{sec:experiments}: component interventions and credit diagnostics, generality, sensitivity, and extended training, followed by computational cost. Reward-interface comparisons appear in Appendix~\ref{app:supplementary-related-work}.

\subsection{Experimental protocols}
\label{app:experimental-protocols}

\subsubsection{Training data construction and selection}
\label{app:training-data}

\paragraph{Source data and serialization.}
We constructed the training corpus by converting the original Infinity-Doc2-5M \texttt{Document\_Parsing/doc2json} annotations into full-page Markdown~\citep{huang2026infinityparser2}. Equations are written in \LaTeX{}, and tables are written in HTML. Appendix~\ref{app:prompt} provides the shared SFT and RL parsing prompt.
\paragraph{SFT data selection.}
We randomly sampled 700K and 50K page-level records from the transformed corpus as subsets for supervised fine-tuning. These subsets provide the SFT initialization settings reported in Tables ~\ref{tab:all-benchmarks} and ~\ref{tab:initialization-backbone}.
\paragraph{RL data selection.}
We used an expert model to re-annotate the transformed corpus and compared its annotations with the corresponding original labels. We ranked records by annotation agreement as a proxy for label quality and retained the 100,000 records with the highest agreement. Expert annotations were used solely for quality assessment and sample selection; RL training retained the original annotations. We sampled 12,800 and 64,000 records from this pool for the 100-update and 500-update experiments, respectively.

\subsubsection{Document parsing prompt}
\label{app:prompt}

The following prompt is used for both supervised fine-tuning and reinforcement-learning rollouts. Our preprocessing converts the \texttt{Document\_Parsing/doc2json} annotations into targets that follow this serialization.

\begin{lstlisting}[language=Python,caption={Full-page document parsing prompt.}]
DOC_PARSE_PROMPT = (
    "\nExtract all readable content from the image in natural human reading "
    "order and output the result as a single Markdown document. Format formulas "
    "as LaTeX. Format tables as HTML: <table>...</table>. Transcribe all other "
    "text as standard Markdown. Preserve the original text without translation "
    "or paraphrasing."
)
\end{lstlisting}

\subsubsection{Training configuration and run provenance}
\label{app:training-config}

Training rewards construct optimization coefficients; benchmark metrics evaluate the resulting transcriptions. Table~\ref{tab:training-config} specifies the configuration for P2TPO and its component variants. Within each setting, methods share the initialization, data and evaluation pipelines, sampling settings, and training budget. Baseline-specific normalization, importance weighting, and clipping follow Appendix~\ref{app:comparison-methods}; the P2TPO policy objective is specified in Appendix~\ref{app:method-details}.

\begin{table}[h]
\centering
\caption{Training configuration for P2TPO and its component variants, which share optimizer and loss settings. Baseline-specific policy objectives are described in Appendix~\ref{app:comparison-methods}.}
\label{tab:training-config}
\footnotesize
\setlength{\tabcolsep}{5pt}
\renewcommand{\arraystretch}{1.12}
\begin{tabularx}{\linewidth}{@{}>{\raggedright\arraybackslash}p{0.27\linewidth}X@{}}
\toprule
Setting & Specification \\
\midrule
Training framework & veRL with Megatron actor training and asynchronous vLLM generation; BF16. \\
RL budget & 100 updates; 128 page prompts per update; $G=8$ responses per prompt. \\
PPO minibatch & 128 prompts, expanded to 1,024 responses; one PPO epoch. \\
Sequence limits & 8,192 prompt tokens (including image tokens) and 8,192 response tokens; total model length 16,384. \\
Rollout sampling & Temperature 1.0, top-$p=1$, top-$k$ disabled. \\
Optimizer & Megatron Adam; learning rate $10^{-6}$, moments $(0.9,0.999)$, weight decay $0.01$, constant schedule, no warmup. \\
Clipping and reduction & Gradient norm limit 1.0; PPO ratio bounds $[0.8,1.2]$ and dual-clip factor 3; global response-token mean. \\
Numerical ratio safeguard & Clamp the log importance ratio to $[-20,20]$ before exponentiation. \\
Regularization & No KL reward or loss; entropy coefficient 0. \\
Memory and parallelism & Offloading disabled; vLLM GPU-memory utilization 0.6; tensor, context, and pipeline parallel sizes 1. \\
Hardware & High-performance computing cards: 8 for RL and 64 for SFT; also used for evaluation. \\
Default credit parameters & $(\lambda,\beta,\omega)=(0.1,0.5,0.2)$; pre-rebate cap $D=8$; dispersion tolerance $\epsilon=10^{-8}$; SCR thresholds $0.8/0.8/0.8$ for text/formula/table; error runs separated by at most $g=3$ intervening tokens are merged. \\
\bottomrule
\end{tabularx}
\end{table}

\paragraph{Run aggregation.}
Primary RL Overall scores and type-specific parsing metrics are summarized across five seeds. Each additional initialization, model-scale, and backbone setting also uses five RL runs, varying the RL data seed. The two UAP ablations in Table~\ref{tab:mechanism-robustness} use the same five paired seeds as full P2TPO.

\paragraph{Study configurations.}
Table~\ref{tab:study-configurations} distinguishes the experimental series. All trained comparisons use 128 page prompts per update and the common data, sampling, and optimizer settings in Table~\ref{tab:training-config}, with method-specific objectives and the stated interventions. The 100- and 500-update studies share an SFT initialization where indicated, but use separate RL datasets and runs.

\begin{table}[h]
\centering
\caption{Configuration of each experimental series. SFT counts specify initialization data; RL counts specify training records. ``All splits'' denotes the six benchmarks and 13 splits in Table~\ref{tab:all-benchmarks}. Overall scores are summarized by means and standard deviations across five runs.}
\label{tab:study-configurations}
\footnotesize
\setlength{\tabcolsep}{3pt}
\renewcommand{\arraystretch}{1.16}
\begin{tabularx}{\linewidth}{@{}>{\raggedright\arraybackslash}p{0.19\linewidth}>{\raggedright\arraybackslash}p{0.25\linewidth}>{\raggedright\arraybackslash}p{0.13\linewidth}>{\raggedright\arraybackslash}p{0.12\linewidth}X@{}}
\toprule
Study & Initialization & RL records / updates & Summary & Evaluation \\
\midrule
Primary comparison and UAP/SCR combinations & Qwen3.5-2B; 700K SFT & 12,800 / 100 & $\mu\pm\sigma$; 5 runs & All splits; content and omission diagnostics on OmniDocBench \\
Internal UAP interventions & Qwen3.5-2B; 700K SFT & 12,800 / 100 & $\mu\pm\sigma$; 5 paired runs & All splits \\
SCR allocation controls & Qwen3.5-2B; 700K SFT & 12,800 / 100 & $\mu\pm\sigma$; 5 runs & All splits \\
Initialization and scale & Qwen3.5-2B and 4B; same 50K SFT manifest & 12,800 / 100 & $\mu\pm\sigma$; 5 data seeds & All splits \\
Additional backbone & Qwen3-VL-2B-Instruct; 700K SFT & 12,800 / 100 & $\mu\pm\sigma$; 5 data seeds & All splits \\
Parameter sensitivity & Qwen3.5-2B; 700K SFT & 12,800 / 100 & $\mu\pm\sigma$; 5 runs & All splits \\
Extended training & Qwen3.5-2B; 50K SFT & 64,000 / 500 & $\mu\pm\sigma$; 5 runs & OmniDocBench v1.6 every ten updates; other 12 splits every 50 \\
\bottomrule
\end{tabularx}
\par\vspace{3pt}
{\scriptsize\raggedright
\textit{Note.} The extended-training series is separate from the 100-update comparisons. Frozen-response diagnostics involve no training update and are specified separately in Appendix~\ref{app:qualitative-records}.\par}
\end{table}

\subsubsection{Comparison methods}
\label{app:comparison-methods}

The evaluated response-level baselines are GRPO, GSPO, GDPO, and RDPO; \blockmethod{}, +UAP, and +SCR are component controls (Table~\ref{tab:all-benchmarks}). 

\paragraph{Response-level baselines.}
\textbf{GRPO}~\citep{shao2024deepseekmath} standardizes scalar page rewards within each prompt's rollout group and broadcasts the resulting advantage across each response. \textbf{GSPO}~\citep{zheng2025gspo} retains group-relative response advantages while using length-normalized sequence importance ratios and sequence-level clipping. \textbf{GDPO}~\citep{liu2026gdpo} normalizes each reward dimension separately within the rollout group, aggregates the normalized dimensions, and applies batch-wise advantage normalization. \textbf{RDPO}~\citep{bai2026rdpo} applies magnitude-aware quantile normalization followed by Mahalanobis whitening within active reward subspaces before aggregation. Each evaluated interface produces one advantage per sampled response.

\paragraph{Component ablations.}
\blockmethod{} assigns relative advantages to matched reference blocks alongside page feedback. Its +UAP variant adds outcome completion, consensus feedback, and omission routing; +SCR redistributes matched-block credit w/o UAP. Full P2TPO combines both components (Appendix~\ref{app:credit-pseudocode}).

\subsubsection{Benchmark protocols and score conventions}
\label{app:external}

We evaluate frozen checkpoints w/o benchmark-specific fine-tuning. OmniDocBench v1.6, Real5-OmniDocBench, Wild-OmniDocBench, DocPTBench, and PureDocBench all use the repository's OmniDocBench v1.6 evaluation engine. FinixDocBench alone uses its official Markdown evaluator through the repository adapter. Prompts, decoding, preprocessing, and post-processing are fixed across methods within each benchmark. All headline scores range from 0 to 100, with higher values better.

\paragraph{Datasets and evaluated scope.}

\begin{table}[h]
\centering
\caption{Benchmark scope and evaluation engines. Page counts specify the evaluated sets. Engine names identify the shared OmniDocBench v1.6 engine or the official FinixDocBench Markdown evaluator.}
\label{tab:benchmark-protocols}
\small
\setlength{\tabcolsep}{3pt}
\renewcommand{\arraystretch}{1.12}
\begin{tabularx}{\linewidth}{@{}>{\raggedright\arraybackslash}p{0.27\linewidth}>{\raggedright\arraybackslash}p{0.14\linewidth}>{\raggedright\arraybackslash}X>{\raggedright\arraybackslash}p{0.16\linewidth}@{}}
\toprule
Benchmark & Pages & Evaluated Scope & Engine \\
\midrule
OmniDocBench v1.6~\citep{ouyang2024omnidocbench,omnidoc2026release} & 1,651 & Official v1.6 ground truth, including 1,355 v1.5 pages. & OmniDocBench v1.6 \\
Real5-OmniDocBench~\citep{zhou2026real5} & 1,355 per condition & Scanning, Warping, Skew, Illumination, Screen photo; shared v1.5 ground truth. & OmniDocBench v1.6 \\
Wild-OmniDocBench~\citep{li2026wildomni} & 1,350 & Real captures of a subset of v1.5 pages. & OmniDocBench v1.6 \\
DocPTBench SynReal~\citep{du2025docptbench,docpt2025protocol} & 1,381 & Photographed set: simulated and real captures. & OmniDocBench v1.6 \\
DocPTBench Dewarp~\citep{du2025docptbench,docpt2025protocol} & 1,381 & Unwarping set: rectified Photographed images. & OmniDocBench v1.6 \\
PureDocBench~\citep{li2026puredocbench,pure2026protocol} & 1,475 per condition & Clean, Digital, Real; ground truth uses the v1.6 page schema. & OmniDocBench v1.6 \\
FinixDocBench~\citep{wang2026finixdoc,finix2026release} & 542 & Track 1 Digital (242) and Track 2 Photo (300). & FinixDocBench \\
\bottomrule
\end{tabularx}
\end{table}

We evaluate the complete published DocPTBench Photographed (SynReal) and Unwarping (Dewarp) sets, each containing 1,381 images~\citep{du2025docptbench,docpt2025protocol}. SynReal comprises 981 simulated captures and 400 real photographs of 100 source documents under four conditions; Dewarp rectifies the entire Photographed set. Only parsing is evaluated: translation metrics such as BLEU, chrF, METEOR, and STEDS are not used.

FinixDocBench combines the 242 Digital pages and 300 Photo pages into one 542-page evaluation set; Track 3 Huge is excluded. We report pooled metrics for this combined set, rather than separate track scores or an equal-weight mean of two track-level scores. Real5-OmniDocBench and Wild-OmniDocBench reuse OmniDocBench source pages and should be interpreted as related robustness evaluations.

\paragraph{Metric definitions and aggregation.}
\label{app:diagnostics}

For each benchmark evaluated with the OmniDocBench v1.6 engine, the headline composite is
\begin{equation}
 \mathrm{Overall}_{\mathrm{OmniDocBench}}
 =\frac{100}{3}\bigl[(1-E_{\mathrm{text}})+S_{\mathrm{CDM}}+S_{\mathrm{TEDS}}\bigr],
 \label{eq:eval-overall}
\end{equation}
where all three input metrics lie in $[0,1]$~\citep{omnidoc2026release}. Text edit distance first pools edit operations within an eligible page, dividing by the sum of maximum prediction/reference lengths, then averages across eligible pages. CDM and TEDS first average element scores within a page and then macro-average across eligible pages. The report reads the evaluator's page-level CDM and TEDS outputs, rather than the raw global element averages. Ground-truth formula/table pages w/o scored records receive zero page similarity.

DocPTBench images and annotations are passed through this same engine. Its Overall is the composite in Eq.~(\ref{eq:eval-overall}), not a numerical conversion of a native DocPTBench headline score. The $1-E_{\mathrm{text}}$ term is part of the OmniDocBench composite definition. Formula edit distance and reading-order edit distance retain their lower-is-better directions and are reported separately; they do not enter this composite. Table~\ref{tab:component-diagnostics} presents CDM and TEDS on the normalized $[0,1]$ scale.

For FinixDocBench, the official evaluator and adapter use
\begin{equation}
 S_{\mathrm{FinixDocBench}}=\frac{100(1-E_{\mathrm{text}})+100(1-E_{\mathrm{order}})+S_{\mathrm{TEDS},100}}{3},
 \label{eq:eval-finix}
\end{equation}
with the adapter clipping the result to $[0,100]$~\citep{finix2026release}. Text and text-only reading-order edit distances are macro-averaged over applicable page groups after within-page pooling; TEDS is averaged over table samples and reported on a $[0,100]$ scale. FinixDocBench therefore replaces the formula component with reading-order accuracy. Headline scores are compared between methods within a benchmark, not directly across benchmarks with different composites.

\paragraph{Omission diagnostics and uncertainty.}
Element omission is the fraction of eligible reference elements that are missing. Local omission is the fraction of omitted reference characters within matched text blocks and excludes whole-element omissions. Table~\ref{tab:component-diagnostics} reports observed omission rates. Across-seed standard deviations describe variability in final evaluation scores. The metric definitions and aggregation rules above specify the inspected evaluation implementation.

\subsection{Component ablations}
\label{app:mechanism-study}

Table~\ref{tab:all-benchmarks} compares reference-element routing and the UAP/SCR combinations; Table~\ref{tab:mechanism-robustness} reports the complete intervention results as means $\pm$ standard deviations under the configurations in Table~\ref{tab:study-configurations}. This section supplements those results with intervention definitions, protocols, and detailed comparisons.

\paragraph{Reference configurations.}
Full P2TPO, w/o UAP (\blockmethod{} + SCR), and w/o SCR (\blockmethod{} + UAP) reproduce Table~\ref{tab:all-benchmarks}. The w/o UAP configuration retains SCR while removing outcome completion, consensus feedback, and omission routing; all other configurations retain outcome completion. Scores and standard deviations are rounded to two decimal places, while score differences are computed before display rounding.

\subsubsection{Internal UAP interventions}
\label{app:uap-ablations}

\paragraph{Ablation protocol.}
The UAP ablations in Table~\ref{tab:mechanism-robustness} follow the primary configuration in Table~\ref{tab:study-configurations}, using full P2TPO's five paired seeds and keeping the mixing weight $\omega$ fixed.
\begin{itemize}[leftmargin=*,nosep]
    \item \textbf{P2TPO w/o consensus feedback ($\lambda=0$).} Setting $\lambda=0$ removes the advantage assigned at imperfect reward consensus. Outcome completion, relative normalization for non-consensus groups, omission routing, and SCR are retained.
    \item \textbf{P2TPO w/o omission routing.} Missing reference elements retain zero-valued outcomes in the group statistics, while their token-level advantage contributions are omitted from Equation~\ref{eq:composition}. The removed contributions are not redistributed; consensus feedback and SCR remain enabled.
\end{itemize}

\paragraph{Feedback activation and omission delivery.}
Full P2TPO has higher means than the no-consensus and no-routing variants on eight and 11 of 13 splits, respectively (Table~\ref{tab:mechanism-robustness}). These comparisons assess imperfect-consensus feedback and delivery of missing-element coefficients as distinct parts of UAP. The higher means without consensus feedback occur on Real5-OmniDocBench Scanning, Illumination, and Screen photo and both DocPTBench splits; without omission routing, they occur on Scanning and SynReal. The evidence supports benefits across most tested splits, with split-dependent effects.

\begin{samepage}
\subsubsection{SCR allocation controls}
\label{app:scr-allocation-controls}

Removing SCR ($\beta=0$) eliminates both nonuniform weighting and its alignment with token correctness. We therefore compare aligned SCR with two controls that retain the redistribution strength while changing where credit is assigned. The existing full P2TPO and P2TPO w/o SCR (\blockmethod{} + UAP) results in Table~\ref{tab:all-benchmarks} provide the aligned and uniform reference configurations.\par
\end{samepage}

\paragraph{Experimental protocol.}
Both controls follow the primary five-run configuration in Table~\ref{tab:study-configurations}, retaining UAP, page feedback, and the remaining training settings. Perturbations operate only on valid tokens within eligible matched blocks, after SCR has computed its final coefficients. The original mask selects error spans for negative block credit and correct positions for positive block credit.

\begin{itemize}[leftmargin=*,nosep]
    \item \textbf{Circular shift.} We jointly rotate the selected-token mask and its coefficient sequence by an offset sampled uniformly from $\{0,\ldots,L-1\}$, where $L$ is the valid-token count. This preserves cyclic order and spacing while randomizing alignment with the original token labels. The zero offset is included.
    \item \textbf{Random support.} We sample the same number of valid positions uniformly w/o replacement. Existing selected and unselected coefficient values are copied to the new support and its complement, respectively. This randomizes token correspondence and disrupts the original cyclic span structure.
\end{itemize}

Both controls preserve the selected-token count and complete multiset of final coefficient values within each block, including the cap-and-rebate results. Hence $\sum_t\widetilde d_t=\sum_t d_t=L$, and the signed and absolute source-block credit budgets are unchanged. Reference-derived correct/error labels and the response-validity mask remain fixed; only the SCR support and its associated coefficients are reassigned among valid tokens within the same block. A fixed mask seed makes assignments repeatable for each rollout block at a given training step, using a separate random stream from rollout sampling.

Comparing aligned SCR with these controls examines the contribution of alignment-informed placement at the same per-block coefficient budget; comparing circular shifts with random support examines the role of cyclic span structure. The uniform reference measures the effect of removing redistribution altogether.

\paragraph{Results and analysis.}
The SCR columns of Table~\ref{tab:mechanism-robustness} compare the two controls with aligned SCR and uniform allocation. Differences below are in mean Overall-score points.
\begin{itemize}[leftmargin=*,nosep]
    \item \textbf{Alignment-informed placement.} The largest gains over random support occur on PureDocBench Real (+1.00), FinixDocBench (+0.89), and PureDocBench Clean (+0.84); over circular shift, the largest gains occur on PureDocBench Digital/Clean (+0.95/+0.83).
    \item \textbf{Comparison with uniform allocation.} Together with the cross-split counts in Section~\ref{sec:complementarity}, these comparisons show that preserving coefficient values and selected-token counts generally does not reproduce aligned SCR's gains. Alignment-informed placement contributes beyond nonuniform weighting alone.
    \item \textbf{Control-specific patterns.} Circular shifts outperform random support in mean scores across nine splits. Additionally, they surpass aligned SCR on FinixDocBench (+0.06) and on DocPTBench SynReal/Dewarp (+0.42/+0.18), respectively. Furthermore, random support exceeds aligned SCR on Dewarp (+0.07). Among these configurations, uniform allocation yields the highest mean score on Real5-OmniDocBench Scanning and DocPTBench Dewarp.
\end{itemize}

\subsection{Credit-assignment diagnostics}
\label{app:qualitative-records}

\begin{figure}[ht]
\centering
\includegraphics[width=\linewidth]{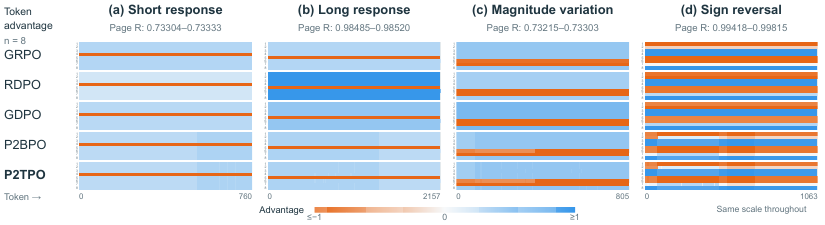}
\caption{\textbf{Frozen responses / no policy update.} Final token-wise advantages on four selected saved rollout groups. Methods are evaluated on common responses; each method row contains eight rollout bands. Horizontal position indexes generated tokens, and panel headers give page-reward ranges. Red/blue denote negative/positive advantages on a shared zero-centered scale, with colors saturating beyond $\pm1$. Panels show (a,b) shorter/longer responses, (c) magnitude variation across rollouts, and (d) within-response sign changes.}
\label{fig:credit-examples}
\end{figure}

The responses in Figure~\ref{fig:credit-examples} were sampled from a frozen checkpoint. Its panels~(c,d) supply the heatmaps in Figure~\ref{fig:credit-sign-reversal}(a,b), respectively; the corresponding lower profiles show responses R6 and R7. P2TPO uses the default credit parameters in Table~\ref{tab:training-config}, and RDPO's covariance is fitted to each source batch.

\paragraph{Output-level diagnostics.}
Table~\ref{tab:component-diagnostics} evaluates trained outputs, complementing the coefficient comparisons on fixed outputs in Figure~\ref{fig:credit-examples}. Appendix~\ref{app:diagnostics} defines the content and omission metrics and their uncertainty reporting.

\subsection{Generality across initialization, scale, and backbone}
\label{app:generality-efficiency}

Table~\ref{tab:initialization-backbone-full} supplements Table~\ref{tab:initialization-backbone} with pre-RL scores and standard deviations for the three additional model and SFT initialization settings. Each method is evaluated over five RL runs using distinct data seeds under the configurations in Table~\ref{tab:study-configurations}. The two comparisons in which GDPO has the higher mean are Qwen3.5-2B/50K on DocPTBench SynReal (76.47 versus 76.31) and Qwen3-VL-2B-Instruct/700K on PureDocBench Digital (76.57 versus 76.51).

\paragraph{Initialization and model scale.}
The Qwen3.5-2B experiments use initial checkpoints obtained by SFT on 50K or 700K pages. For the model-scale experiments, Qwen3.5-2B and Qwen3.5-4B checkpoints are prepared using the same 50K-page SFT manifest. Each checkpoint defines a separate setting for comparing RL algorithms, assessing their applicability across initial policies and model scales.

\paragraph{Generality across backbones.}
Qwen3-VL-2B-Instruct receives document SFT on 700K pages to obtain its initial checkpoint, followed by the same 100-update RL budget. Comparing RL methods from this shared checkpoint evaluates their applicability to another backbone.

\begin{table}[h]
\centering
\caption{Cross-benchmark RL evaluation across model scales and SFT-initialized checkpoints. Metrics and splits follow Table~\ref{tab:all-benchmarks}. Within each panel, Init reports pre-RL scores and RL methods share the starting checkpoint and 100-update budget.}
\label{tab:initialization-backbone-full}
\scriptsize
\setlength{\tabcolsep}{2pt}
\setlength{\aboverulesep}{1.5pt}
\setlength{\belowrulesep}{1.5pt}
\renewcommand{\arraystretch}{1.12}
\begin{tabularx}{\linewidth}{@{}>{\hspace*{0.7em}}ll*{6}{>{\centering\arraybackslash}X}@{}}
\toprule
\textbf{Benchmark} & \textbf{Split} & \textbf{Init} & \textbf{GRPO} & \textbf{GSPO} & \textbf{GDPO} & \textbf{RDPO} & \textbf{P2TPO} \\
\midrule
\rowcolor{black!5}[0pt][0pt]
\multicolumn{8}{@{}l@{}}{\makebox[\linewidth][l]{\rule{0pt}{2.6ex}(a) \textbf{Qwen3.5-2B}\quad SFT data: \textbf{50K pages}\hspace*{0.6em}}} \\
OmniDocBench & {\hspace*{0.4em}\rmfamily v1.6} & 90.28 & \meanstd{91.00}{0.44} & \meanstd{90.39}{0.22} & \meanstd{\underline{91.05}}{0.36} & \meanstd{90.85}{0.25} & \meanstd{\mathbf{91.62}}{0.30} \\
\cmidrule[0.2pt](lr){2-8}
\multirow{3}{*}{PureDocBench} & {\hspace*{0.4em}\rmfamily Clean} & 80.60 & \meanstd{80.83}{1.19} & \meanstd{81.01}{0.53} & \meanstd{\underline{81.02}}{1.00} & \meanstd{79.75}{1.15} & \meanstd{\mathbf{81.98}}{0.62} \\
 & {\hspace*{0.4em}\rmfamily Digital} & 76.20 & \meanstd{\underline{78.26}}{1.30} & \meanstd{78.11}{0.44} & \meanstd{77.85}{0.77} & \meanstd{77.27}{1.08} & \meanstd{\mathbf{79.16}}{0.59} \\
 & {\hspace*{0.4em}\rmfamily Real} & 64.64 & \meanstd{\underline{69.82}}{1.42} & \meanstd{68.66}{0.70} & \meanstd{69.58}{0.81} & \meanstd{69.20}{0.86} & \meanstd{\mathbf{71.01}}{0.72} \\
\cmidrule[0.2pt](lr){2-8}
FinixDocBench & {\hspace*{0.4em}\rmfamily ---} & 69.46 & \meanstd{69.45}{1.93} & \meanstd{70.11}{0.36} & \meanstd{\underline{70.95}}{1.12} & \meanstd{\underline{70.95}}{2.89} & \meanstd{\mathbf{73.01}}{1.72} \\
\cmidrule[0.2pt](lr){2-8}
\multirow{5}{*}{Real5-OmniDocBench} & {\hspace*{0.4em}\rmfamily Scanning} & 89.27 & \meanstd{90.66}{0.37} & \meanstd{90.23}{0.12} & \meanstd{\underline{90.97}}{0.23} & \meanstd{\underline{90.97}}{0.30} & \meanstd{\mathbf{91.29}}{0.25} \\
 & {\hspace*{0.4em}\rmfamily Warping} & 87.66 & \meanstd{\underline{89.56}}{0.33} & \meanstd{89.06}{0.26} & \meanstd{89.47}{0.50} & \meanstd{89.53}{0.42} & \meanstd{\mathbf{89.86}}{0.19} \\
 & {\hspace*{0.4em}\rmfamily Skew} & 85.35 & \meanstd{88.02}{0.75} & \meanstd{87.04}{0.35} & \meanstd{\underline{88.11}}{0.48} & \meanstd{87.61}{0.68} & \meanstd{\mathbf{88.44}}{0.29} \\
 & {\hspace*{0.4em}\rmfamily Illumination} & 89.51 & \meanstd{90.81}{0.49} & \meanstd{90.35}{0.33} & \meanstd{\underline{91.10}}{0.28} & \meanstd{90.83}{0.24} & \meanstd{\mathbf{91.23}}{0.27} \\
 & {\hspace*{0.4em}\rmfamily Screen photo} & 88.45 & \meanstd{89.79}{0.44} & \meanstd{89.36}{0.33} & \meanstd{\underline{90.43}}{0.37} & \meanstd{90.24}{0.33} & \meanstd{\mathbf{90.50}}{0.21} \\
\cmidrule[0.2pt](lr){2-8}
\multirow{2}{*}{DocPTBench} & {\hspace*{0.4em}\rmfamily SynReal} & 74.19 & \meanstd{75.57}{0.54} & \meanstd{75.21}{0.39} & \meanstd{\mathbf{76.47}}{0.48} & \meanstd{76.23}{1.02} & \meanstd{\underline{76.31}}{0.43} \\
 & {\hspace*{0.4em}\rmfamily Dewarp} & 81.72 & \meanstd{82.43}{0.79} & \meanstd{82.47}{0.36} & \meanstd{\underline{83.26}}{0.71} & \meanstd{82.90}{0.79} & \meanstd{\mathbf{83.63}}{0.40} \\
\cmidrule[0.2pt](lr){2-8}
Wild-OmniDocBench & {\hspace*{0.4em}\rmfamily ---} & 80.60 & \meanstd{83.67}{0.22} & \meanstd{82.99}{0.37} & \meanstd{83.62}{0.41} & \meanstd{\underline{84.27}}{0.62} & \meanstd{\mathbf{84.48}}{0.31} \\
\midrule
\rowcolor{black!5}[0pt][0pt]
\multicolumn{8}{@{}l@{}}{\makebox[\linewidth][l]{\rule{0pt}{2.6ex}(b) \textbf{Qwen3.5-4B}\quad SFT data: \textbf{50K pages}\hspace*{0.6em}}} \\
OmniDocBench & {\hspace*{0.4em}\rmfamily v1.6} & 91.45 & \meanstd{\underline{92.13}}{0.41} & \meanstd{91.99}{0.29} & \meanstd{92.07}{0.21} & \meanstd{91.54}{0.70} & \meanstd{\mathbf{92.53}}{0.26} \\
\cmidrule[0.2pt](lr){2-8}
\multirow{3}{*}{PureDocBench} & {\hspace*{0.4em}\rmfamily Clean} & 83.50 & \meanstd{\underline{84.41}}{0.41} & \meanstd{84.28}{0.44} & \meanstd{84.07}{0.37} & \meanstd{82.90}{1.32} & \meanstd{\mathbf{84.95}}{0.48} \\
 & {\hspace*{0.4em}\rmfamily Digital} & 79.46 & \meanstd{\underline{81.80}}{0.62} & \meanstd{81.06}{0.38} & \meanstd{81.47}{0.20} & \meanstd{80.14}{1.45} & \meanstd{\mathbf{82.34}}{0.48} \\
 & {\hspace*{0.4em}\rmfamily Real} & 68.77 & \meanstd{73.55}{0.53} & \meanstd{72.52}{0.44} & \meanstd{\underline{73.61}}{0.26} & \meanstd{71.86}{1.77} & \meanstd{\mathbf{74.48}}{0.48} \\
\cmidrule[0.2pt](lr){2-8}
FinixDocBench & {\hspace*{0.4em}\rmfamily ---} & 71.55 & \meanstd{74.33}{1.65} & \meanstd{\underline{74.43}}{0.49} & \meanstd{73.59}{1.37} & \meanstd{69.38}{7.14} & \meanstd{\mathbf{76.74}}{1.49} \\
\cmidrule[0.2pt](lr){2-8}
\multirow{5}{*}{Real5-OmniDocBench} & {\hspace*{0.4em}\rmfamily Scanning} & 91.18 & \meanstd{\underline{92.04}}{0.34} & \meanstd{91.89}{0.22} & \meanstd{91.81}{0.18} & \meanstd{91.52}{0.60} & \meanstd{\mathbf{92.23}}{0.10} \\
 & {\hspace*{0.4em}\rmfamily Warping} & 89.41 & \meanstd{91.02}{0.36} & \meanstd{90.84}{0.43} & \meanstd{\underline{91.18}}{0.25} & \meanstd{90.22}{0.73} & \meanstd{\mathbf{91.45}}{0.44} \\
 & {\hspace*{0.4em}\rmfamily Skew} & 89.64 & \meanstd{\underline{90.95}}{0.23} & \meanstd{90.40}{0.31} & \meanstd{90.78}{0.30} & \meanstd{90.02}{0.67} & \meanstd{\mathbf{91.33}}{0.33} \\
 & {\hspace*{0.4em}\rmfamily Illumination} & 91.13 & \meanstd{92.48}{0.39} & \meanstd{92.33}{0.13} & \meanstd{\underline{92.59}}{0.14} & \meanstd{91.71}{0.66} & \meanstd{\mathbf{92.80}}{0.30} \\
 & {\hspace*{0.4em}\rmfamily Screen photo} & 90.15 & \meanstd{\underline{91.91}}{0.26} & \meanstd{91.71}{0.23} & \meanstd{91.87}{0.23} & \meanstd{91.08}{0.77} & \meanstd{\mathbf{92.16}}{0.28} \\
\cmidrule[0.2pt](lr){2-8}
\multirow{2}{*}{DocPTBench} & {\hspace*{0.4em}\rmfamily SynReal} & 73.86 & \meanstd{\underline{77.86}}{0.86} & \meanstd{76.51}{0.73} & \meanstd{77.76}{0.72} & \meanstd{76.37}{1.22} & \meanstd{\mathbf{78.19}}{0.40} \\
 & {\hspace*{0.4em}\rmfamily Dewarp} & 82.55 & \meanstd{84.02}{0.65} & \meanstd{83.98}{0.28} & \meanstd{\underline{84.11}}{0.66} & \meanstd{83.23}{1.25} & \meanstd{\mathbf{84.88}}{0.58} \\
\cmidrule[0.2pt](lr){2-8}
Wild-OmniDocBench & {\hspace*{0.4em}\rmfamily ---} & 82.93 & \meanstd{\underline{86.04}}{0.78} & \meanstd{85.34}{0.30} & \meanstd{85.83}{0.37} & \meanstd{85.10}{0.77} & \meanstd{\mathbf{86.28}}{0.20} \\
\midrule
\rowcolor{black!5}[0pt][0pt]
\multicolumn{8}{@{}l@{}}{\makebox[\linewidth][l]{\rule{0pt}{2.6ex}(c) \textbf{Qwen3-VL-2B-Instruct}\quad SFT data: \textbf{700K pages}\hspace*{0.6em}}} \\
OmniDocBench & {\hspace*{0.4em}\rmfamily v1.6} & 90.27 & \meanstd{90.45}{0.24} & \meanstd{90.55}{0.16} & \meanstd{\underline{90.57}}{0.23} & \meanstd{88.25}{1.17} & \meanstd{\mathbf{91.29}}{0.09} \\
\cmidrule[0.2pt](lr){2-8}
\multirow{3}{*}{PureDocBench} & {\hspace*{0.4em}\rmfamily Clean} & 78.32 & \meanstd{\underline{81.81}}{0.20} & \meanstd{81.80}{0.24} & \meanstd{81.75}{0.30} & \meanstd{79.67}{0.89} & \meanstd{\mathbf{81.95}}{0.33} \\
 & {\hspace*{0.4em}\rmfamily Digital} & 71.77 & \meanstd{76.08}{0.29} & \meanstd{76.29}{0.13} & \meanstd{\mathbf{76.57}}{0.20} & \meanstd{73.98}{1.07} & \meanstd{\underline{76.51}}{0.28} \\
 & {\hspace*{0.4em}\rmfamily Real} & 59.72 & \meanstd{65.45}{0.45} & \meanstd{65.44}{0.30} & \meanstd{\underline{65.76}}{0.25} & \meanstd{63.53}{0.96} & \meanstd{\mathbf{66.25}}{0.29} \\
\cmidrule[0.2pt](lr){2-8}
FinixDocBench & {\hspace*{0.4em}\rmfamily ---} & 64.55 & \meanstd{67.68}{0.63} & \meanstd{67.75}{0.61} & \meanstd{\underline{67.95}}{0.43} & \meanstd{65.18}{1.41} & \meanstd{\mathbf{68.23}}{0.24} \\
\cmidrule[0.2pt](lr){2-8}
\multirow{5}{*}{Real5-OmniDocBench} & {\hspace*{0.4em}\rmfamily Scanning} & 89.11 & \meanstd{\underline{90.07}}{0.21} & \meanstd{89.42}{0.21} & \meanstd{89.89}{0.32} & \meanstd{87.79}{0.80} & \meanstd{\mathbf{90.46}}{0.32} \\
 & {\hspace*{0.4em}\rmfamily Warping} & 84.41 & \meanstd{85.19}{0.75} & \meanstd{84.17}{0.47} & \meanstd{\underline{85.62}}{0.67} & \meanstd{84.46}{0.43} & \meanstd{\mathbf{86.37}}{0.52} \\
 & {\hspace*{0.4em}\rmfamily Skew} & 76.94 & \meanstd{76.83}{1.48} & \meanstd{77.74}{0.62} & \meanstd{\underline{77.82}}{0.95} & \meanstd{76.66}{1.52} & \meanstd{\mathbf{79.79}}{1.24} \\
 & {\hspace*{0.4em}\rmfamily Illumination} & 87.81 & \meanstd{89.26}{0.35} & \meanstd{88.79}{0.36} & \meanstd{\underline{89.38}}{0.45} & \meanstd{87.15}{0.55} & \meanstd{\mathbf{89.65}}{0.23} \\
 & {\hspace*{0.4em}\rmfamily Screen photo} & 87.42 & \meanstd{\underline{88.46}}{0.46} & \meanstd{88.09}{0.30} & \meanstd{88.39}{0.46} & \meanstd{86.01}{1.05} & \meanstd{\mathbf{89.33}}{0.35} \\
\cmidrule[0.2pt](lr){2-8}
\multirow{2}{*}{DocPTBench} & {\hspace*{0.4em}\rmfamily SynReal} & 62.70 & \meanstd{\underline{67.80}}{0.67} & \meanstd{67.18}{0.85} & \meanstd{67.23}{0.97} & \meanstd{65.91}{1.41} & \meanstd{\mathbf{68.94}}{0.59} \\
 & {\hspace*{0.4em}\rmfamily Dewarp} & 77.02 & \meanstd{\underline{78.28}}{0.54} & \meanstd{77.38}{0.23} & \meanstd{77.97}{0.67} & \meanstd{76.16}{0.42} & \meanstd{\mathbf{79.34}}{0.53} \\
\cmidrule[0.2pt](lr){2-8}
Wild-OmniDocBench & {\hspace*{0.4em}\rmfamily ---} & 76.47 & \meanstd{79.20}{0.58} & \meanstd{78.58}{0.13} & \meanstd{\underline{79.65}}{0.40} & \meanstd{77.78}{0.38} & \meanstd{\mathbf{80.32}}{0.19} \\
\bottomrule
\end{tabularx}
\par\vspace{3pt}
{\scriptsize\raggedright
\textit{Note.} Init reports the measured pre-RL score of the shared SFT-initialized checkpoint; no standard deviation is reported. RL entries are mean $\pm$ standard deviation over five runs with distinct RL data seeds; best and second-best RL means are bold and underlined, respectively.
All panels use 100 RL updates with 128 page prompts per update (12,800 page-prompt presentations). The Qwen3.5-2B/700K setting is reported in Table~\ref{tab:all-benchmarks}.
\par}
\end{table}

\subsection{Robustness to hyperparameter choices}
\label{app:parameter-sensitivity}

We vary one parameter at a time around the default P2TPO configuration $(\lambda,\beta,\omega)=(0.1,0.5,0.2)$ while holding the other two fixed (Figure~\ref{fig:parameter-sensitivity}; Table~\ref{tab:parameter-sensitivity}). Here $\lambda$ controls imperfect-consensus feedback, $\beta$ controls SCR concentration, and $\omega$ controls the local contribution relative to page feedback. Entries report means $\pm$ standard deviations over five seeds; the default column reproduces Table~\ref{tab:all-benchmarks}.

\begin{figure}[h]
\centering
\includegraphics[width=\linewidth]{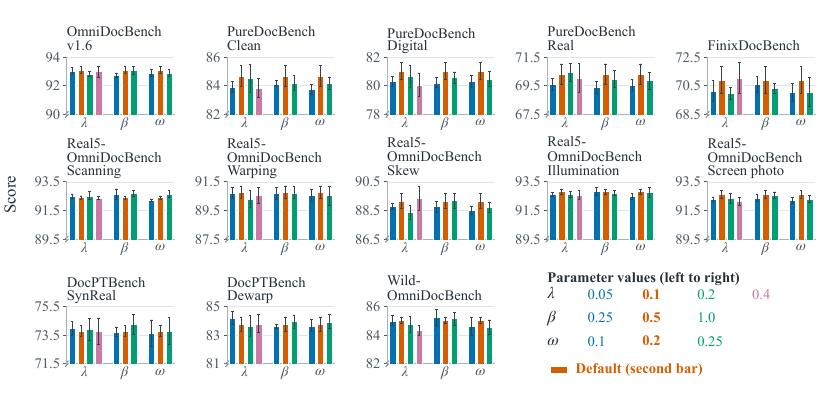}
\caption{One-at-a-time sensitivity of P2TPO to the imperfect-consensus strength $\lambda$, span-concentration strength $\beta$, and page--local mixing weight $\omega$. In each sweep, the other two parameters remain at their default values. Bars show five-seed means; error bars show standard deviations.}
\label{fig:parameter-sensitivity}
\end{figure}

\begin{table}[h]
\centering
\caption{One-at-a-time parameter sensitivity. Benchmark names and splits follow Table~\ref{tab:all-benchmarks}. The default column uses $(\lambda,\beta,\omega)=(0.1,0.5,0.2)$. The $\lambda$, $\beta$, and $\omega$ columns vary the indicated parameter while keeping the other two at their default values.}
\label{tab:parameter-sensitivity}
\label{tab:sensitivity-lambda}
\label{tab:sensitivity-beta}
\label{tab:sensitivity-omega}
\scriptsize
\setlength{\tabcolsep}{1pt}
\renewcommand{\arraystretch}{1.08}
\resizebox{\linewidth}{!}{%
\begin{tabular}{@{}l@{\hspace{3pt}}l@{\hspace{3pt}}*{8}{c}@{}}
\toprule
& & \multicolumn{1}{c}{Default} & \multicolumn{3}{c}{$\lambda$ sweep} & \multicolumn{2}{c}{$\beta$ sweep} & \multicolumn{2}{c}{$\omega$ sweep} \\
\cmidrule(lr){3-3}\cmidrule(lr){4-6}\cmidrule(lr){7-8}\cmidrule(lr){9-10}
Benchmark & Split & $(0.1,0.5,0.2)$ & $0.05$ & $0.2$ & $0.4$ & $0.25$ & $1.0$ & $0.1$ & $0.25$ \\
\midrule
OmniDocBench & v1.6 & \meanstd{93.12}{0.23} & \meanstd{93.02}{0.24} & \meanstd{92.84}{0.19} & \meanstd{93.00}{0.38} & \meanstd{92.75}{0.13} & \meanstd{93.07}{0.28} & \meanstd{92.88}{0.25} & \meanstd{92.92}{0.23} \\
\multirow{3}{*}{PureDocBench} & Clean & \meanstd{84.68}{0.74} & \meanstd{83.92}{0.38} & \meanstd{84.51}{0.99} & \meanstd{83.84}{0.67} & \meanstd{84.15}{0.21} & \meanstd{84.22}{0.52} & \meanstd{83.74}{0.34} & \meanstd{84.17}{0.45} \\
 & Digital & \meanstd{81.03}{0.60} & \meanstd{80.31}{0.33} & \meanstd{80.68}{0.76} & \meanstd{80.08}{0.82} & \meanstd{80.21}{0.35} & \meanstd{80.58}{0.39} & \meanstd{80.30}{0.44} & \meanstd{80.48}{0.53} \\
 & Real & \meanstd{70.33}{0.70} & \meanstd{69.61}{0.42} & \meanstd{70.44}{0.62} & \meanstd{70.04}{1.00} & \meanstd{69.39}{0.40} & \meanstd{69.98}{0.59} & \meanstd{69.54}{0.41} & \meanstd{69.87}{0.59} \\
FinixDocBench & --- & \meanstd{70.92}{0.92} & \meanstd{70.14}{0.74} & \meanstd{69.97}{0.42} & \meanstd{71.07}{1.09} & \meanstd{70.62}{0.56} & \meanstd{70.34}{0.34} & \meanstd{70.03}{0.61} & \meanstd{70.07}{0.99} \\
\multirow{5}{*}{\shortstack[l]{Real5-\\OmniDocBench}} & Scanning & \meanstd{92.40}{0.12} & \meanstd{92.49}{0.17} & \meanstd{92.53}{0.28} & \meanstd{92.37}{0.12} & \meanstd{92.65}{0.36} & \meanstd{92.70}{0.21} & \meanstd{92.22}{0.09} & \meanstd{92.66}{0.23} \\
 & Warping & \meanstd{90.77}{0.45} & \meanstd{90.74}{0.41} & \meanstd{90.29}{0.59} & \meanstd{90.58}{0.55} & \meanstd{90.70}{0.44} & \meanstd{90.74}{0.42} & \meanstd{90.56}{0.45} & \meanstd{90.54}{0.65} \\
 & Skew & \meanstd{89.15}{0.52} & \meanstd{88.78}{0.22} & \meanstd{88.37}{0.49} & \meanstd{89.36}{0.84} & \meanstd{88.78}{0.38} & \meanstd{89.20}{0.52} & \meanstd{88.51}{0.31} & \meanstd{88.72}{0.36} \\
 & Illumination & \meanstd{92.83}{0.17} & \meanstd{92.63}{0.15} & \meanstd{92.62}{0.21} & \meanstd{92.59}{0.34} & \meanstd{92.82}{0.27} & \meanstd{92.71}{0.18} & \meanstd{92.51}{0.20} & \meanstd{92.79}{0.35} \\
 & Screen photo & \meanstd{92.62}{0.29} & \meanstd{92.31}{0.14} & \meanstd{92.34}{0.36} & \meanstd{92.14}{0.25} & \meanstd{92.39}{0.24} & \meanstd{92.57}{0.23} & \meanstd{92.21}{0.21} & \meanstd{92.32}{0.24} \\
\multirow{2}{*}{DocPTBench} & SynReal & \meanstd{73.79}{0.39} & \meanstd{73.96}{0.49} & \meanstd{73.90}{0.76} & \meanstd{73.77}{0.93} & \meanstd{73.73}{0.34} & \meanstd{74.26}{0.71} & \meanstd{73.65}{0.91} & \meanstd{73.80}{0.94} \\
 & Dewarp & \meanstd{83.76}{0.51} & \meanstd{84.22}{0.46} & \meanstd{83.63}{0.75} & \meanstd{83.80}{0.64} & \meanstd{83.62}{0.15} & \meanstd{83.95}{0.46} & \meanstd{83.66}{0.42} & \meanstd{83.94}{0.48} \\
Wild-OmniDocBench & --- & \meanstd{85.02}{0.21} & \meanstd{85.01}{0.35} & \meanstd{84.74}{0.58} & \meanstd{84.33}{0.32} & \meanstd{85.23}{0.57} & \meanstd{85.16}{0.45} & \meanstd{84.62}{0.63} & \meanstd{84.56}{0.50} \\
\bottomrule
\end{tabular}%
}
\end{table}

The sweeps cover $\lambda\in\{0.05,0.1,0.2,0.4\}$, $\beta\in\{0.25,0.5,1.0\}$, and $\omega\in\{0.1,0.2,0.25\}$. For any fixed split, the maximum difference between tested values is at most 1.10 points for $\lambda$ and 0.94 points for either $\beta$ or $\omega$; every reported mean remains within 0.95 points of its default. All eight distinct RL configurations improve upon their shared pre-RL scores on every split, and each exceeds the strongest response-level baseline on at least 11 of 13 splits (Table~\ref{tab:parameter-sensitivity}). These bounded score changes and retained cross-benchmark gains indicate robustness within the tested one-at-a-time ranges.

\subsection{Extended training over 500 updates}
\label{app:extended-training}

GRPO, GSPO, GDPO, RDPO, and full P2TPO (UAP and SCR) each start from the same Qwen3.5-2B checkpoint obtained by SFT on 50K pages and undergo 500 RL updates. The extension uses 64,000 RL records, 128 page prompts per update, and five runs per method with distinct RL data seeds. All 13 splits use this initialization and training protocol; OmniDocBench v1.6 is evaluated every ten updates, and the other 12 splits every 50. The measured pre-RL Overall scores are 90.28 on OmniDocBench v1.6, 85.35 on Real5-OmniDocBench Skew, and 74.19 on DocPTBench SynReal; these provide the Init references in Figure~\ref{fig:extended-training}. As summarized in Table~\ref{tab:study-configurations}, this series uses separate RL runs and data from the 100-update experiments in Tables~\ref{tab:all-benchmarks} and~\ref{tab:initialization-backbone}; changes below use its own checkpoint means.
Figure~\ref{fig:extended-training}(a) shows the OmniDocBench v1.6 trajectory; Table~\ref{tab:extended-training} reports selected checkpoints and temporal averages. The additional evaluations span updates 50 through 500. Table~\ref{tab:extended-checkpoints} reports all 120 additional split--checkpoint results, with means and standard deviations across five runs. Metrics and split definitions follow Table~\ref{tab:all-benchmarks}.

\begin{table}[h]
\centering
\caption{Five-run checkpoint scores (mean $\pm$ standard deviation) and temporal averages. $\Delta$ is update 500 minus update 100, computed before rounding. Early and late windows average the ten checkpoint means at updates 110--200 and 410--500, respectively; these summarize repeated evaluations of the same runs.}
\label{tab:extended-training}
\small
\renewcommand{\meanstd}[2]{\ensuremath{#1\,\pm\,#2}}
\setlength{\tabcolsep}{4pt}
\renewcommand{\arraystretch}{1.15}
\begin{tabularx}{\linewidth}{@{}l*{2}{>{\centering\arraybackslash}X}rrr@{}}
\toprule
Method & Update 100 & Update 500 & $\Delta$ & Early window & Late window \\
\midrule
GRPO & \meanstd{90.60}{0.48} & \meanstd{90.14}{0.70} & $-0.46$ & 90.89 & 90.45 \\
GSPO & \meanstd{90.48}{0.28} & \meanstd{90.31}{0.10} & $-0.17$ & 90.66 & 89.95 \\
GDPO & \meanstd{90.99}{0.32} & \meanstd{90.42}{0.65} & $-0.57$ & 90.99 & 90.41 \\
RDPO & \meanstd{91.05}{0.26} & \meanstd{89.91}{1.42} & $-1.14$ & 90.86 & 89.86 \\
P2TPO & \meanstd{\mathbf{91.38}}{0.42} & \meanstd{\mathbf{91.49}}{0.65} & $+0.12$ & \textbf{91.44} & \textbf{91.24} \\
\bottomrule
\end{tabularx}
\end{table}

\paragraph{Performance retention.}
At update 500, P2TPO reaches 91.49 on OmniDocBench v1.6, 1.21 points above Init and a reported 0.12-point gain over update 100 using unrounded scores. It has the highest mean score at 45 of the 50 evaluated checkpoints. Its early- and late-window means are 91.44 and 91.24, compared with decreases of 0.44, 0.71, 0.58, and 0.99 points for GRPO, GSPO, GDPO, and RDPO (Table~\ref{tab:extended-training}). The trajectory is nonmonotonic, ranging from a peak of 91.72 at update 180 to 90.78 at update 410 before the final checkpoint (Figure~\ref{fig:extended-training}(a)). These comparisons support retaining the measured gain under the longer budget.

\paragraph{Performance across additional splits.}
P2TPO has the highest mean at update 500 on all 12 additional splits, with margins of 0.25--1.94 points over the strongest response-level baseline in each split. It strictly leads at 78 of the 120 evaluated split--checkpoint pairs and ties for the highest mean at two more. Figure~\ref{fig:extended-training}(b,c) presents complementary patterns: Real5-OmniDocBench Skew improves from 88.44 to 89.00 between updates 100 and 500, whereas DocPTBench SynReal remains nearly unchanged at 76.34 versus 76.36. Their final margins over the strongest baseline are 1.94 and 1.60 points, respectively. Skew also rises across the final four evaluations (88.12, 88.41, 88.66, and 89.00), while SynReal fluctuates around its earlier performance.

Across the 12 additional splits, P2TPO's update-500 mean is higher than its update-100 mean on seven splits and lower on five. Changes over this interval vary across methods and splits: all five methods improve on FinixDocBench, and GSPO improves on the three PureDocBench splits. Together with Table~\ref{tab:extended-checkpoints}, these trends support the final comparative-performance claim while documenting the nonmonotonic, task-dependent trajectories.

\paragraph{Relation to the component studies.}
The longer-horizon study assesses comparative parsing performance under a 500-update budget. The UAP interventions and SCR allocation controls in Appendix~\ref{app:mechanism-study} separately assess feedback construction and alignment-informed placement; the longer-horizon results complement those component comparisons.

\begingroup
\makeatletter
\renewcommand{\fnum@table}{\normalsize\tablename~\thetable}
\makeatother
\ifdim\linewidth>440pt
  \fontsize{7.5}{9}\selectfont
\else
  \fontsize{6}{7.5}\selectfont
  \renewcommand{\meanstd}[2]{\scalebox{0.84}{\ensuremath{#1{\scriptstyle\,\pm\,#2}}}}
\fi
\setlength{\tabcolsep}{2pt}
\renewcommand{\arraystretch}{1.05}
\setlength{\LTcapwidth}{\linewidth}
\setlength{\LTleft}{\fill}
\setlength{\LTright}{\fill}
\begin{longtable}{@{}l*{10}{c}@{}}
\caption{\normalsize Complete checkpoint results on the 12 additional splits during 500-update Qwen3.5-2B/50K training. Columns denote RL updates. Entries are Overall scores (mean $\pm$ standard deviation over five runs), evaluated every 50 updates. Bold denotes the highest mean at each checkpoint, determined before rounding.}
\label{tab:extended-checkpoints}\\
\toprule
Method & 50 & 100 & 150 & 200 & 250 & 300 & 350 & 400 & 450 & 500 \\
\midrule
\endfirsthead
\caption[]{\normalsize Complete checkpoint results on the 12 additional splits (continued).}\\
\toprule
Method & 50 & 100 & 150 & 200 & 250 & 300 & 350 & 400 & 450 & 500 \\
\midrule
\endhead
\multicolumn{11}{r@{}}{\textit{Continued on next page}} \\
\endfoot
\bottomrule
\endlastfoot
\multicolumn{11}{@{}l}{\textit{PureDocBench: Clean}} \\*
GRPO & \meanstd{\mathbf{80.69}}{0.60} & \meanstd{80.64}{1.14} & \meanstd{80.54}{0.60} & \meanstd{79.88}{0.90} & \meanstd{80.23}{0.48} & \meanstd{79.99}{1.06} & \meanstd{80.26}{1.58} & \meanstd{80.49}{0.79} & \meanstd{79.93}{2.10} & \meanstd{79.59}{1.37} \\* %
GSPO & \meanstd{79.86}{0.30} & \meanstd{79.97}{0.63} & \meanstd{80.24}{0.49} & \meanstd{80.49}{0.59} & \meanstd{80.17}{0.83} & \meanstd{80.93}{1.02} & \meanstd{79.70}{1.39} & \meanstd{79.80}{0.96} & \meanstd{78.99}{1.16} & \meanstd{80.34}{0.42} \\* %
GDPO & \meanstd{80.29}{0.44} & \meanstd{81.03}{2.25} & \meanstd{81.09}{1.23} & \meanstd{\mathbf{81.95}}{0.78} & \meanstd{80.68}{0.63} & \meanstd{79.93}{1.63} & \meanstd{80.96}{1.53} & \meanstd{80.87}{0.81} & \meanstd{78.51}{1.36} & \meanstd{79.37}{1.23} \\* %
RDPO & \meanstd{80.03}{0.67} & \meanstd{80.68}{0.48} & \meanstd{80.58}{0.79} & \meanstd{80.42}{1.04} & \meanstd{80.94}{1.35} & \meanstd{79.68}{0.78} & \meanstd{79.81}{0.42} & \meanstd{78.91}{1.30} & \meanstd{78.85}{2.06} & \meanstd{78.97}{0.46} \\* %
P2TPO & \meanstd{80.51}{0.52} & \meanstd{\mathbf{81.72}}{0.80} & \meanstd{\mathbf{81.45}}{0.28} & \meanstd{80.36}{0.82} & \meanstd{\mathbf{81.20}}{0.39} & \meanstd{\mathbf{81.24}}{0.61} & \meanstd{\mathbf{81.06}}{0.71} & \meanstd{\mathbf{81.69}}{0.77} & \meanstd{\mathbf{82.02}}{1.50} & \meanstd{\mathbf{81.28}}{0.38} \\ %
\midrule
\multicolumn{11}{@{}l}{\textit{PureDocBench: Digital}} \\*
GRPO & \meanstd{77.41}{0.67} & \meanstd{77.62}{1.35} & \meanstd{77.82}{0.66} & \meanstd{77.75}{0.80} & \meanstd{77.77}{0.64} & \meanstd{77.37}{1.23} & \meanstd{77.40}{1.63} & \meanstd{77.53}{1.01} & \meanstd{77.59}{1.83} & \meanstd{77.06}{1.19} \\* %
GSPO & \meanstd{76.63}{0.50} & \meanstd{76.65}{0.73} & \meanstd{77.06}{0.43} & \meanstd{77.79}{0.41} & \meanstd{77.43}{0.90} & \meanstd{78.28}{0.98} & \meanstd{76.18}{1.54} & \meanstd{76.94}{1.04} & \meanstd{76.35}{1.05} & \meanstd{77.43}{0.40} \\* %
GDPO & \meanstd{77.56}{0.45} & \meanstd{78.67}{2.21} & \meanstd{78.18}{1.36} & \meanstd{\mathbf{79.09}}{0.66} & \meanstd{77.94}{0.54} & \meanstd{77.17}{1.41} & \meanstd{\mathbf{78.34}}{1.76} & \meanstd{78.24}{0.72} & \meanstd{75.43}{1.36} & \meanstd{77.00}{1.11} \\* %
RDPO & \meanstd{77.42}{0.87} & \meanstd{78.56}{0.40} & \meanstd{77.69}{0.62} & \meanstd{77.46}{1.10} & \meanstd{78.26}{1.34} & \meanstd{77.20}{0.54} & \meanstd{76.48}{0.60} & \meanstd{75.98}{1.21} & \meanstd{76.05}{1.93} & \meanstd{76.56}{0.56} \\* %
P2TPO & \meanstd{\mathbf{77.58}}{0.79} & \meanstd{\mathbf{79.26}}{0.93} & \meanstd{\mathbf{78.82}}{0.23} & \meanstd{77.80}{0.80} & \meanstd{\mathbf{78.64}}{0.46} & \meanstd{\mathbf{78.66}}{1.06} & \meanstd{78.15}{0.53} & \meanstd{\mathbf{78.77}}{0.72} & \meanstd{\mathbf{79.09}}{1.44} & \meanstd{\mathbf{78.27}}{0.34} \\ %
\midrule
\multicolumn{11}{@{}l}{\textit{PureDocBench: Real}} \\*
GRPO & \meanstd{69.26}{0.43} & \meanstd{69.77}{0.85} & \meanstd{69.93}{0.55} & \meanstd{69.23}{0.52} & \meanstd{70.25}{0.22} & \meanstd{69.71}{0.87} & \meanstd{69.03}{1.91} & \meanstd{69.86}{0.59} & \meanstd{69.73}{2.15} & \meanstd{69.08}{1.40} \\* %
GSPO & \meanstd{67.52}{0.42} & \meanstd{67.41}{0.66} & \meanstd{68.50}{0.57} & \meanstd{69.12}{0.26} & \meanstd{68.82}{0.80} & \meanstd{69.24}{1.16} & \meanstd{67.65}{1.46} & \meanstd{68.20}{1.18} & \meanstd{68.19}{0.69} & \meanstd{68.64}{0.54} \\* %
GDPO & \meanstd{\mathbf{69.32}}{0.36} & \meanstd{69.98}{1.96} & \meanstd{70.19}{0.95} & \meanstd{\mathbf{71.31}}{0.70} & \meanstd{70.16}{0.67} & \meanstd{70.06}{1.39} & \meanstd{\mathbf{70.70}}{1.55} & \meanstd{69.75}{0.77} & \meanstd{67.44}{1.58} & \meanstd{68.86}{1.14} \\* %
RDPO & \meanstd{68.65}{0.00} & \meanstd{69.77}{0.93} & \meanstd{69.96}{1.11} & \meanstd{69.22}{1.46} & \meanstd{70.30}{1.42} & \meanstd{69.60}{0.60} & \meanstd{68.57}{0.77} & \meanstd{67.93}{1.26} & \meanstd{68.15}{1.45} & \meanstd{68.94}{0.75} \\* %
P2TPO & \meanstd{69.16}{0.42} & \meanstd{\mathbf{70.60}}{0.75} & \meanstd{\mathbf{70.74}}{0.42} & \meanstd{70.23}{1.03} & \meanstd{\mathbf{70.78}}{0.51} & \meanstd{\mathbf{70.99}}{1.05} & \meanstd{70.52}{0.91} & \meanstd{\mathbf{70.94}}{0.84} & \meanstd{\mathbf{70.99}}{1.09} & \meanstd{\mathbf{70.49}}{0.53} \\ %
\midrule
\multicolumn{11}{@{}l}{\textit{FinixDocBench}} \\*
GRPO & \meanstd{69.09}{0.45} & \meanstd{69.71}{0.70} & \meanstd{70.47}{0.56} & \meanstd{69.88}{1.18} & \meanstd{69.90}{1.01} & \meanstd{69.44}{1.56} & \meanstd{70.22}{0.45} & \meanstd{70.84}{0.48} & \meanstd{\mathbf{71.25}}{1.20} & \meanstd{71.38}{1.20} \\* %
GSPO & \meanstd{68.74}{0.55} & \meanstd{68.87}{0.06} & \meanstd{69.31}{0.50} & \meanstd{70.13}{0.43} & \meanstd{70.64}{0.94} & \meanstd{70.87}{1.19} & \meanstd{68.65}{0.86} & \meanstd{69.13}{1.39} & \meanstd{68.77}{1.25} & \meanstd{69.60}{1.09} \\* %
GDPO & \meanstd{69.35}{0.51} & \meanstd{69.18}{1.66} & \meanstd{69.39}{1.28} & \meanstd{\mathbf{71.71}}{1.06} & \meanstd{69.35}{2.11} & \meanstd{69.79}{1.93} & \meanstd{70.63}{1.27} & \meanstd{70.41}{0.86} & \meanstd{69.29}{0.28} & \meanstd{69.59}{1.19} \\* %
RDPO & \meanstd{\mathbf{69.46}}{0.80} & \meanstd{69.76}{1.28} & \meanstd{69.18}{1.05} & \meanstd{69.74}{1.59} & \meanstd{\mathbf{71.64}}{0.73} & \meanstd{71.18}{1.47} & \meanstd{68.29}{1.48} & \meanstd{69.41}{0.61} & \meanstd{69.74}{1.27} & \meanstd{69.80}{0.88} \\* %
P2TPO & \meanstd{\mathbf{69.46}}{0.42} & \meanstd{\mathbf{71.13}}{0.72} & \meanstd{\mathbf{71.86}}{0.89} & \meanstd{70.60}{1.02} & \meanstd{71.12}{0.35} & \meanstd{\mathbf{72.18}}{1.18} & \meanstd{\mathbf{72.18}}{1.87} & \meanstd{\mathbf{70.95}}{0.50} & \meanstd{71.19}{1.19} & \meanstd{\mathbf{71.63}}{1.28} \\ %
\midrule
\multicolumn{11}{@{}l}{\textit{DocPTBench: SynReal}} \\*
GRPO & \meanstd{75.83}{0.39} & \meanstd{75.89}{0.97} & \meanstd{75.41}{1.33} & \meanstd{75.03}{1.09} & \meanstd{76.18}{0.35} & \meanstd{75.01}{1.50} & \meanstd{73.85}{2.83} & \meanstd{74.11}{1.34} & \meanstd{74.87}{2.50} & \meanstd{74.32}{1.57} \\* %
GSPO & \meanstd{74.56}{0.42} & \meanstd{74.86}{0.54} & \meanstd{75.17}{0.89} & \meanstd{74.69}{0.83} & \meanstd{73.37}{1.04} & \meanstd{74.54}{1.23} & \meanstd{72.15}{1.48} & \meanstd{72.28}{1.31} & \meanstd{72.85}{0.57} & \meanstd{73.67}{1.50} \\* %
GDPO & \meanstd{\mathbf{76.33}}{0.11} & \meanstd{\mathbf{76.99}}{1.39} & \meanstd{75.97}{0.61} & \meanstd{\mathbf{76.03}}{1.09} & \meanstd{75.31}{0.82} & \meanstd{75.36}{1.74} & \meanstd{\mathbf{75.95}}{2.47} & \meanstd{74.73}{1.21} & \meanstd{72.93}{2.45} & \meanstd{74.26}{1.75} \\* %
RDPO & \meanstd{76.28}{0.59} & \meanstd{76.80}{0.69} & \meanstd{75.40}{0.51} & \meanstd{74.94}{1.03} & \meanstd{75.27}{0.80} & \meanstd{74.21}{0.53} & \meanstd{73.51}{1.88} & \meanstd{72.58}{2.58} & \meanstd{74.72}{0.98} & \meanstd{74.76}{1.03} \\* %
P2TPO & \meanstd{75.85}{0.24} & \meanstd{76.34}{0.45} & \meanstd{\mathbf{76.48}}{0.58} & \meanstd{75.57}{0.93} & \meanstd{\mathbf{76.29}}{1.01} & \meanstd{\mathbf{76.20}}{0.89} & \meanstd{75.51}{0.85} & \meanstd{\mathbf{76.25}}{0.29} & \meanstd{\mathbf{76.57}}{1.53} & \meanstd{\mathbf{76.36}}{1.28} \\ %
\midrule
\multicolumn{11}{@{}l}{\textit{DocPTBench: Dewarp}} \\*
GRPO & \meanstd{82.87}{0.34} & \meanstd{82.31}{0.79} & \meanstd{82.46}{1.26} & \meanstd{82.43}{0.79} & \meanstd{82.88}{0.52} & \meanstd{82.02}{1.06} & \meanstd{81.51}{1.86} & \meanstd{81.77}{0.48} & \meanstd{80.58}{2.03} & \meanstd{81.62}{1.57} \\* %
GSPO & \meanstd{80.99}{0.64} & \meanstd{81.41}{0.52} & \meanstd{81.89}{0.37} & \meanstd{82.62}{1.05} & \meanstd{82.44}{0.39} & \meanstd{\mathbf{82.46}}{0.68} & \meanstd{80.61}{1.35} & \meanstd{80.55}{1.17} & \meanstd{80.65}{0.37} & \meanstd{80.99}{0.78} \\* %
GDPO & \meanstd{\mathbf{83.22}}{0.21} & \meanstd{83.48}{0.77} & \meanstd{\mathbf{83.49}}{0.54} & \meanstd{\mathbf{83.63}}{0.42} & \meanstd{82.69}{0.42} & \meanstd{82.09}{1.67} & \meanstd{\mathbf{82.74}}{1.35} & \meanstd{81.93}{0.54} & \meanstd{80.71}{1.73} & \meanstd{80.97}{1.61} \\* %
RDPO & \meanstd{82.85}{0.37} & \meanstd{83.28}{0.52} & \meanstd{82.71}{0.66} & \meanstd{82.02}{0.68} & \meanstd{82.21}{1.35} & \meanstd{81.16}{0.88} & \meanstd{80.77}{1.05} & \meanstd{80.26}{1.63} & \meanstd{81.56}{0.93} & \meanstd{80.76}{1.38} \\* %
P2TPO & \meanstd{82.93}{0.59} & \meanstd{\mathbf{83.78}}{0.78} & \meanstd{83.13}{0.43} & \meanstd{83.01}{0.82} & \meanstd{\mathbf{82.95}}{0.84} & \meanstd{82.41}{0.38} & \meanstd{82.65}{0.54} & \meanstd{\mathbf{82.54}}{0.30} & \meanstd{\mathbf{83.15}}{1.32} & \meanstd{\mathbf{83.24}}{1.09} \\ %
\midrule
\multicolumn{11}{@{}l}{\textit{Real5-OmniDocBench: Scanning}} \\*
GRPO & \meanstd{90.83}{0.22} & \meanstd{90.63}{0.25} & \meanstd{90.67}{0.48} & \meanstd{90.53}{0.39} & \meanstd{90.70}{0.26} & \meanstd{90.40}{0.25} & \meanstd{89.13}{1.08} & \meanstd{90.41}{0.38} & \meanstd{89.75}{1.08} & \meanstd{89.77}{0.71} \\* %
GSPO & \meanstd{89.98}{0.22} & \meanstd{90.10}{0.21} & \meanstd{90.36}{0.14} & \meanstd{90.07}{0.66} & \meanstd{90.17}{0.57} & \meanstd{89.88}{0.52} & \meanstd{88.83}{0.97} & \meanstd{89.27}{0.61} & \meanstd{89.63}{0.15} & \meanstd{89.42}{0.36} \\* %
GDPO & \meanstd{\mathbf{91.44}}{0.16} & \meanstd{\mathbf{91.28}}{0.43} & \meanstd{\mathbf{91.17}}{0.21} & \meanstd{90.98}{0.36} & \meanstd{90.68}{0.38} & \meanstd{90.44}{0.65} & \meanstd{90.33}{0.61} & \meanstd{90.16}{0.30} & \meanstd{89.76}{0.87} & \meanstd{89.94}{0.91} \\* %
RDPO & \meanstd{90.98}{0.22} & \meanstd{91.25}{0.31} & \meanstd{90.68}{0.31} & \meanstd{90.06}{0.48} & \meanstd{89.87}{0.66} & \meanstd{89.29}{0.82} & \meanstd{89.65}{0.79} & \meanstd{89.70}{0.54} & \meanstd{89.70}{0.46} & \meanstd{89.27}{0.35} \\* %
P2TPO & \meanstd{90.87}{0.32} & \meanstd{91.21}{0.41} & \meanstd{90.83}{0.27} & \meanstd{\mathbf{91.18}}{0.60} & \meanstd{\mathbf{90.95}}{0.21} & \meanstd{\mathbf{90.77}}{0.24} & \meanstd{\mathbf{90.42}}{0.24} & \meanstd{\mathbf{90.74}}{0.08} & \meanstd{\mathbf{90.78}}{0.57} & \meanstd{\mathbf{91.40}}{0.75} \\ %
\midrule
\multicolumn{11}{@{}l}{\textit{Real5-OmniDocBench: Warping}} \\*
GRPO & \meanstd{89.09}{0.35} & \meanstd{89.08}{0.58} & \meanstd{89.40}{0.67} & \meanstd{88.99}{0.50} & \meanstd{89.93}{0.35} & \meanstd{89.07}{0.47} & \meanstd{88.53}{1.19} & \meanstd{89.07}{0.37} & \meanstd{88.59}{1.36} & \meanstd{88.43}{0.93} \\* %
GSPO & \meanstd{87.84}{0.41} & \meanstd{88.43}{0.41} & \meanstd{88.78}{0.28} & \meanstd{88.41}{0.75} & \meanstd{88.51}{0.49} & \meanstd{89.00}{0.48} & \meanstd{87.01}{0.91} & \meanstd{87.59}{0.76} & \meanstd{87.95}{0.75} & \meanstd{87.64}{0.42} \\* %
GDPO & \meanstd{\mathbf{89.47}}{0.06} & \meanstd{\mathbf{89.93}}{0.31} & \meanstd{\mathbf{89.76}}{0.31} & \meanstd{\mathbf{90.03}}{0.34} & \meanstd{89.68}{0.64} & \meanstd{88.85}{0.95} & \meanstd{89.32}{1.00} & \meanstd{89.31}{0.25} & \meanstd{87.99}{1.40} & \meanstd{88.52}{1.25} \\* %
RDPO & \meanstd{88.93}{0.52} & \meanstd{89.39}{0.26} & \meanstd{89.01}{0.38} & \meanstd{88.81}{0.71} & \meanstd{88.98}{0.40} & \meanstd{88.10}{0.53} & \meanstd{88.35}{0.77} & \meanstd{87.53}{1.23} & \meanstd{88.21}{0.89} & \meanstd{88.01}{0.45} \\* %
P2TPO & \meanstd{89.25}{0.29} & \meanstd{89.82}{0.49} & \meanstd{89.75}{0.38} & \meanstd{89.49}{0.61} & \meanstd{\mathbf{89.99}}{0.61} & \meanstd{\mathbf{89.96}}{0.31} & \meanstd{\mathbf{89.99}}{0.57} & \meanstd{\mathbf{89.68}}{0.25} & \meanstd{\mathbf{89.71}}{0.66} & \meanstd{\mathbf{90.00}}{0.88} \\ %
\midrule
\multicolumn{11}{@{}l}{\textit{Real5-OmniDocBench: Skew}} \\*
GRPO & \meanstd{87.79}{0.41} & \meanstd{87.76}{0.43} & \meanstd{87.56}{1.22} & \meanstd{87.63}{0.95} & \meanstd{\mathbf{88.26}}{0.43} & \meanstd{87.54}{0.78} & \meanstd{86.72}{1.20} & \meanstd{87.26}{0.35} & \meanstd{86.73}{1.51} & \meanstd{87.06}{0.96} \\* %
GSPO & \meanstd{86.09}{0.56} & \meanstd{86.51}{0.24} & \meanstd{86.16}{0.42} & \meanstd{86.31}{0.90} & \meanstd{86.29}{0.62} & \meanstd{86.47}{0.81} & \meanstd{84.83}{1.11} & \meanstd{85.99}{0.56} & \meanstd{86.22}{0.41} & \meanstd{86.10}{0.62} \\* %
GDPO & \meanstd{\mathbf{88.28}}{0.20} & \meanstd{88.39}{0.78} & \meanstd{\mathbf{88.52}}{0.12} & \meanstd{\mathbf{88.81}}{0.25} & \meanstd{88.16}{0.43} & \meanstd{87.35}{0.85} & \meanstd{87.88}{1.25} & \meanstd{86.80}{0.38} & \meanstd{85.88}{1.42} & \meanstd{86.64}{1.42} \\* %
RDPO & \meanstd{88.08}{0.59} & \meanstd{88.39}{0.57} & \meanstd{87.69}{0.37} & \meanstd{87.55}{0.49} & \meanstd{87.57}{0.47} & \meanstd{86.75}{0.64} & \meanstd{86.90}{1.12} & \meanstd{85.69}{1.74} & \meanstd{86.25}{1.48} & \meanstd{85.81}{0.92} \\* %
P2TPO & \meanstd{87.64}{0.20} & \meanstd{\mathbf{88.44}}{0.69} & \meanstd{88.40}{0.61} & \meanstd{87.84}{0.79} & \meanstd{\mathbf{88.26}}{0.55} & \meanstd{\mathbf{88.16}}{0.55} & \meanstd{\mathbf{88.12}}{0.50} & \meanstd{\mathbf{88.41}}{0.26} & \meanstd{\mathbf{88.66}}{0.79} & \meanstd{\mathbf{89.00}}{0.77} \\ %
\midrule
\multicolumn{11}{@{}l}{\textit{Real5-OmniDocBench: Illumination}} \\*
GRPO & \meanstd{90.22}{0.30} & \meanstd{90.47}{0.39} & \meanstd{90.79}{0.63} & \meanstd{90.24}{0.67} & \meanstd{90.85}{0.21} & \meanstd{90.36}{0.50} & \meanstd{90.09}{1.07} & \meanstd{\mathbf{90.64}}{0.47} & \meanstd{89.65}{0.94} & \meanstd{89.87}{0.85} \\* %
GSPO & \meanstd{89.42}{0.20} & \meanstd{89.58}{0.18} & \meanstd{90.36}{0.41} & \meanstd{90.03}{0.46} & \meanstd{90.31}{0.49} & \meanstd{90.47}{0.43} & \meanstd{88.68}{0.74} & \meanstd{88.99}{0.71} & \meanstd{89.75}{0.16} & \meanstd{89.57}{0.35} \\* %
GDPO & \meanstd{\mathbf{90.79}}{0.19} & \meanstd{90.84}{0.35} & \meanstd{90.90}{0.38} & \meanstd{\mathbf{91.10}}{0.39} & \meanstd{91.22}{0.61} & \meanstd{90.09}{0.98} & \meanstd{90.49}{1.02} & \meanstd{90.37}{0.21} & \meanstd{89.49}{0.94} & \meanstd{90.35}{0.66} \\* %
RDPO & \meanstd{90.46}{0.17} & \meanstd{91.00}{0.48} & \meanstd{90.75}{0.27} & \meanstd{90.56}{0.57} & \meanstd{90.78}{0.49} & \meanstd{90.17}{0.21} & \meanstd{90.32}{0.61} & \meanstd{89.15}{0.44} & \meanstd{89.78}{0.36} & \meanstd{90.01}{0.21} \\* %
P2TPO & \meanstd{90.32}{0.36} & \meanstd{\mathbf{91.38}}{0.42} & \meanstd{\mathbf{91.28}}{0.18} & \meanstd{90.97}{0.59} & \meanstd{\mathbf{91.32}}{0.49} & \meanstd{\mathbf{90.90}}{0.26} & \meanstd{\mathbf{90.87}}{0.47} & \meanstd{90.44}{0.27} & \meanstd{\mathbf{90.84}}{0.55} & \meanstd{\mathbf{90.85}}{0.55} \\ %
\midrule
\multicolumn{11}{@{}l}{\textit{Real5-OmniDocBench: Screen photo}} \\*
GRPO & \meanstd{89.94}{0.27} & \meanstd{89.82}{0.40} & \meanstd{89.97}{0.57} & \meanstd{89.75}{0.54} & \meanstd{90.04}{0.15} & \meanstd{89.81}{0.48} & \meanstd{89.17}{1.20} & \meanstd{90.06}{0.39} & \meanstd{89.39}{1.19} & \meanstd{89.66}{0.86} \\* %
GSPO & \meanstd{88.74}{0.41} & \meanstd{89.19}{0.22} & \meanstd{89.59}{0.26} & \meanstd{89.58}{0.41} & \meanstd{89.16}{0.24} & \meanstd{89.29}{0.29} & \meanstd{87.88}{0.82} & \meanstd{88.67}{0.71} & \meanstd{88.60}{0.36} & \meanstd{88.74}{0.22} \\* %
GDPO & \meanstd{\mathbf{90.18}}{0.25} & \meanstd{\mathbf{90.44}}{0.64} & \meanstd{89.98}{0.40} & \meanstd{\mathbf{90.55}}{0.30} & \meanstd{90.18}{0.23} & \meanstd{89.65}{0.95} & \meanstd{89.92}{0.83} & \meanstd{89.80}{0.10} & \meanstd{88.53}{1.24} & \meanstd{89.97}{1.04} \\* %
RDPO & \meanstd{89.90}{0.33} & \meanstd{90.06}{0.15} & \meanstd{89.89}{0.43} & \meanstd{89.66}{0.64} & \meanstd{89.75}{0.47} & \meanstd{89.28}{0.47} & \meanstd{89.02}{0.85} & \meanstd{88.25}{0.93} & \meanstd{89.22}{0.52} & \meanstd{89.34}{0.18} \\* %
P2TPO & \meanstd{89.74}{0.14} & \meanstd{90.33}{0.31} & \meanstd{\mathbf{90.29}}{0.22} & \meanstd{90.31}{0.48} & \meanstd{\mathbf{90.74}}{0.60} & \meanstd{\mathbf{90.35}}{0.29} & \meanstd{\mathbf{90.25}}{0.40} & \meanstd{\mathbf{90.19}}{0.24} & \meanstd{\mathbf{90.10}}{0.31} & \meanstd{\mathbf{90.49}}{0.33} \\ %
\midrule
\multicolumn{11}{@{}l}{\textit{Wild-OmniDocBench}} \\*
GRPO & \meanstd{83.77}{0.36} & \meanstd{83.36}{0.33} & \meanstd{83.44}{0.95} & \meanstd{83.49}{0.63} & \meanstd{\mathbf{84.38}}{0.79} & \meanstd{83.45}{0.68} & \meanstd{82.55}{1.24} & \meanstd{83.33}{0.51} & \meanstd{82.19}{1.19} & \meanstd{82.70}{0.64} \\* %
GSPO & \meanstd{81.70}{0.28} & \meanstd{82.35}{0.45} & \meanstd{83.05}{0.12} & \meanstd{82.63}{0.96} & \meanstd{82.64}{0.56} & \meanstd{82.88}{0.54} & \meanstd{81.26}{1.32} & \meanstd{81.81}{1.17} & \meanstd{82.51}{1.33} & \meanstd{82.09}{1.01} \\* %
GDPO & \meanstd{\mathbf{84.32}}{0.44} & \meanstd{\mathbf{84.30}}{0.76} & \meanstd{\mathbf{84.47}}{0.33} & \meanstd{\mathbf{84.40}}{0.55} & \meanstd{83.98}{0.49} & \meanstd{83.38}{0.89} & \meanstd{83.18}{1.13} & \meanstd{82.94}{0.47} & \meanstd{82.37}{1.23} & \meanstd{83.16}{1.03} \\* %
RDPO & \meanstd{83.71}{0.11} & \meanstd{84.25}{0.60} & \meanstd{83.22}{0.87} & \meanstd{83.13}{0.35} & \meanstd{83.20}{0.46} & \meanstd{82.44}{0.61} & \meanstd{83.21}{0.92} & \meanstd{82.27}{0.68} & \meanstd{82.32}{0.65} & \meanstd{82.20}{0.79} \\* %
P2TPO & \meanstd{83.75}{0.16} & \meanstd{84.00}{0.48} & \meanstd{84.31}{0.35} & \meanstd{84.10}{0.60} & \meanstd{84.29}{0.48} & \meanstd{\mathbf{84.23}}{0.13} & \meanstd{\mathbf{83.94}}{0.48} & \meanstd{\mathbf{84.14}}{0.31} & \meanstd{\mathbf{83.83}}{0.41} & \meanstd{\mathbf{84.36}}{0.65} \\ %
\end{longtable}
\endgroup

\subsection{Computational cost}
\label{app:cost-controls}
\label{app:runtime}
The primary 100-update P2TPO run requires approximately 60 GPU-hours on eight high-performance computing cards. The full credit-assignment mechanism takes 1.215\,s/update, compared with 0.707\,s/update for GRPO, adding 0.51\,s/update, equivalent to approximately 0.19\% of per-update wall time. The recorded advantage-computation shares are 0.465\% for P2TPO and 0.265\% for GRPO. Generation and actor-update timing distributions are comparable to GRPO.

\section{Supplementary Theoretical Analysis}
\label{app:theory-proofs}

We establish three properties of the credit construction: the information lost by aggregation, the feedback restored by UAP, and the conditions under which SCR increases allocation to effective targets. Alignments, token supports, and masks are held fixed. The aggregation model uses common matched rollouts and a fixed additive utility; the UAP and SCR results concern the coefficients defined in Section~\ref{sec:method}.

\subsection{What is lost in aggregation?}
\label{subsec:block-contrasts}

Consider $G\ge2$ rollouts in which all $K$ reference blocks are matched, and let $\mathbf Q\in[0,1]^{G\times K}$ contain their scores. Partition the blocks into $C$ nonempty element groups $\mathcal J_c$, with $1<C<K$. Here a group can contain several blocks of the same element type; it is coarser than an individual reference element. Let $M_{k,c}>0$ for $k\in\mathcal J_c$ and $M_{k,c}=0$ otherwise, with $\sum_{k\in\mathcal J_c}M_{k,c}=1$. Let $\alpha_c>0$ and $\sum_c\alpha_c=1$. Group scores are $\mathbf R=\mathbf Q\mathbf M$ and the additive page utility is $\mathbf r=\mathbf R\boldsymbol\alpha$. Define $\mathbf H=\mathbf I_G-G^{-1}\mathbf1_G\mathbf1_G^\top$ and the unstandardized block contrasts $\mathbf B=\mathbf H\mathbf Q$.

\begin{proposition}[Strict hierarchy of aggregation losses]
\label{prop:block-information}
Block contrasts determine group contrasts $\mathbf H\mathbf R=\mathbf B\mathbf M$, which determine page contrasts $\mathbf H\mathbf r=\mathbf B\mathbf M\boldsymbol\alpha$. Each reverse determination is impossible in general. More precisely, on the space of centered matrices $\mathcal B=\{\mathbf B\in\mathbb R^{G\times K}:\mathbf H\mathbf B=\mathbf B\}$,
\begin{equation}
\{\mathbf0\}\ \subsetneq\
\{\mathbf B\in\mathcal B:\mathbf B\mathbf M=\mathbf0\}
\ \subsetneq\
\{\mathbf B\in\mathcal B:\mathbf B\mathbf M\boldsymbol\alpha=\mathbf0\}.
\label{eq:aggregation-loss-hierarchy}
\end{equation}
Both strict losses occur for scores in $[0,1]$. Consequently, deterministic processing of pooled rewards alone cannot recover all distinctions available at the preceding resolution.
\end{proposition}

\begin{proof}
The contrast identities follow by linearity. For the first strict inclusion, choose two blocks $k,\ell$ in the same group $c$, and a vector $\mathbf x$ with $x_k=M_{\ell,c}$, $x_\ell=-M_{k,c}$, and all other entries zero. Then $\mathbf x\ne0$ but $\mathbf x^\top\mathbf M=0$. For the second, choose two distinct groups $c,d$ and a vector $\mathbf v\in\mathbb R^C$ with $v_c=\alpha_d$, $v_d=-\alpha_c$, and other entries zero. Set $x_k=v_j$ for $k\in\mathcal J_j$. Column normalization gives $\mathbf x^\top\mathbf M=\mathbf v^\top\ne0$ but $\mathbf v^\top\boldsymbol\alpha=0$. In either construction, take a nonzero centered vector $\mathbf u\in\mathbb R^G$ and set $\mathbf B=\mathbf u\mathbf x^\top$. For sufficiently small $\delta>0$, the two score matrices $\mathbf Q_0=\tfrac12\mathbf1_G\mathbf1_K^\top$ and $\mathbf Q_1=\mathbf Q_0+\delta\mathbf B$ lie in $[0,1]^{G\times K}$. In the first construction they have identical group scores but different block contrasts; in the second they have identical page scores but different group contrasts. A deterministic function of these identical pooled inputs cannot distinguish the respective pair.
\end{proof}

This hierarchy concerns scores before standardization. Separate block normalization retains within-block contrasts when their dispersion exceeds the tolerance, but does not preserve their original scales. No information ordering is asserted after independent standardization. The result applies to a page reward only when it has the stated fixed additive form; it does not assume that every evaluation metric or training reward has this form.

\paragraph{Auxiliary routing lemma.}
Fix one rollout with disjoint nonempty block supports of lengths $L_k$, and let $b_k$ be its normalized \blockmethod{} coefficients. Their broadcast token vector on the union of these matched-block supports, before page/local composition, is the reconstruction target. Let $D_{\mathrm B}$, $D_{\mathrm E}$, and $D_{\mathrm P}$ be the minimum squared token errors using unrestricted coefficients per block, per element group, and per page, respectively. Thus
\[
D_{\mathrm E}=\min_{\boldsymbol\gamma}\sum_c\sum_{k\in\mathcal J_c}L_k(b_k-\gamma_c)^2,
\qquad D_{\mathrm P}=\min_v\sum_k L_k(b_k-v)^2.
\]
Write $N_c=\sum_{k\in\mathcal J_c}L_k$, $\bar b_c=N_c^{-1}\sum_{k\in\mathcal J_c}L_kb_k$, and $\bar b=(\sum_kL_k)^{-1}\sum_kL_kb_k$.

\begin{lemma}[Resolution of block, element-group, and page routing]
\label{lem:routing-resolution}
The optimal reconstruction errors satisfy
\begin{equation}
\begin{aligned}
D_{\mathrm B}&=0,\qquad
D_{\mathrm E}=\sum_c\sum_{k\in\mathcal J_c}L_k(b_k-\bar b_c)^2,\\
D_{\mathrm P}&=D_{\mathrm E}+\sum_cN_c(\bar b_c-\bar b)^2.
\end{aligned}
\label{eq:routing-hierarchy}
\end{equation}
Hence $D_{\mathrm B}\le D_{\mathrm E}\le D_{\mathrm P}$. The first inequality is strict exactly when some group contains unequal block coefficients; the second is strict exactly when the group means are not all equal.
\end{lemma}

\begin{proof}
Independent block coefficients reproduce the target. Minimizing each group sum gives $\gamma_c=\bar b_c$, and minimizing the page sum gives $v=\bar b$. Expanding $b_k-\bar b=(b_k-\bar b_c)+(\bar b_c-\bar b)$, the cross terms vanish because $\sum_{k\in\mathcal J_c}L_k(b_k-\bar b_c)=0$. The decomposition and strictness conditions follow from positivity of the lengths.
\end{proof}

Both inequalities can be strict for actual \blockmethod{} coefficients. With equal block lengths of one, groups $\{1,2\}$ and $\{3\}$, and rollout scores $(0.9,0.2,0.9)$ and $(0.2,0.9,0.2)$, population normalization gives $\mathbf b=(1,-1,1)$ in the first rollout when $\epsilon<0.35$. The errors are $0<2<8/3$. A global coefficient cannot distinguish blocks within a shared group, while routing indexed by individual reference blocks matches \blockmethod{}'s resolution. The minima allow free coefficients: a particular reward-normalization rule need not attain them, and its actual reconstruction errors need not follow this ordering. The lemma compares coefficient expressivity, not policy performance or the optimality of the chosen target coefficients.

\subsection{How does UAP restore missing feedback?}
\label{subsec:uap-theory}

For one reference block, consider $n\ge1$ valid rollouts with completed outcomes $Q_i\in[0,c]$, where $c>0$ is fixed and missing outcomes equal zero. Let $\mu$ and $\sigma$ be their mean and population standard deviation. The algorithm uses the unit-ceiling specialization $c=1$. Outcome completion makes omissions available to the comparison, while the ceiling gap supplies feedback when dispersion is too small for relative normalization.

\begin{proposition}[Feedback activation and delivery by UAP]
\label{prop:uap-feedback}
Let $\lambda>0$ and $\epsilon\ge0$. UAP has the following properties.
\begin{enumerate}
\item If $0<m<n$ matched rollouts all have score $s>0$ and the other $n-m$ outcomes are missing, then $\mu=ms/n$ and $\sigma=s\sqrt{m(n-m)}/n$. When $\sigma>\epsilon$,
\begin{equation}
A^{\mathrm U}_{\mathrm{matched}}=\sqrt{\frac{n-m}{m}},\qquad
A^{\mathrm U}_{\mathrm{missing}}=-\sqrt{\frac{m}{n-m}}.
\label{eq:uap-completion-feedback}
\end{equation}
Matched-only centering instead gives zero feedback to all matched rollouts and no coefficient for omissions.
\item Whenever $\sigma\le\epsilon$, every valid rollout receives $A_i^{\mathrm U}=-\lambda(c-\mu)$. This coefficient is negative exactly when $\mu<c$ and vanishes at ceiling quality. In particular, an imperfect tie at $r<c$ receives $-\lambda(c-r)$ and an all-missing group receives $-\lambda c$.
\item A missing outcome with coefficient $a$ and routing weights $\nu_t\ge0$, $\sum_t\nu_t=1$, contributes $a\nu_t$ at generated tokens, preserving its signed and absolute mass:
\begin{equation}
\sum_t a\nu_t=a,\qquad \sum_t|a\nu_t|=|a|.
\label{eq:uap-delivery}
\end{equation}
\end{enumerate}
\end{proposition}

\begin{proof}
For the mixed group, direct evaluation gives $\sigma^2=m(n-m)s^2/n^2$. Substitution of the matched score $s$ and missing score zero into $(Q_i-\mu)/\sigma$ yields Eq.~(\ref{eq:uap-completion-feedback}); the matched-only scores are constant. For the low-dispersion branch, $0\le\mu\le c$ implies $h(\mu,c)=c-\mu$, proving the signs and special cases. Since every $Q_i\le c$, $\mu=c$ holds only when all outcomes equal $c$. More generally, centered outcomes alone cannot distinguish $r\mathbf1_n$ from $c\mathbf1_n$, as both center to zero; UAP retains their distinct uncentered means. Finally, nonnegative normalized routing weights give Eq.~(\ref{eq:uap-delivery}) by summation.
\end{proof}

With unequal matched scores, omissions receive negative coefficients in the relative branch when $\mu>0$, while a matched coefficient is positive only if its score exceeds $\mu$. In the low-dispersion branch, all coefficients share the ceiling-gap value, even if an individual outcome reaches the ceiling. The routing identities hold before loss masking and composition; mass assigned to inactive tokens does not enter the masked objective.

\subsection{When does SCR direct feedback to effective targets?}
\label{subsec:span-scr}

Assume SCR is enabled. Fix an eligible matched block of length $L\ge2$ with coefficient $A\ne0$, redistribution strength $\beta\in[0,1]$, and pre-rebate cap $D>1$. Its selected mask $S$ targets errors when $A<0$ and correct positions when $A>0$. Assume $p=|S|/L\in(0,1)$. Let $\mathcal E$ be an externally specified target set (no statistical independence from $S$ is assumed), $q=|\mathcal E|/L$ its prevalence, and $\varrho=|S\cap\mathcal E|/|S|$ the mask precision. Thus effectiveness is evaluated against the intended target, rather than defined by membership in $S$.

\begin{proposition}[Validity and target allocation of SCR]
\label{prop:scr-concentration}
Algorithm~\ref{alg:p2tpo-scr}'s cap-and-rebate operation yields
\begin{equation}
d_t=1-\eta+\frac{\eta}{p}\mathbf1[t\in S],\qquad
\eta=\min\{\beta,p(D-1+\beta)\}\in[0,\beta].
\label{eq:span-rule}
\end{equation}
The weights are nonnegative and preserve the signed and absolute source-block budgets:
\begin{equation}
\sum_{t=1}^L Ad_t=LA,\qquad \sum_{t=1}^L|Ad_t|=L|A|.
\label{eq:span-conservation}
\end{equation}
The target's share of absolute coefficient mass satisfies
\begin{equation}
F_{\mathcal E}=\frac{\sum_{t\in\mathcal E}|Ad_t|}{L|A|}
=q+\eta(\varrho-q).
\label{eq:precision-gain}
\end{equation}
It exceeds uniform allocation exactly when $\eta>0$ and $\varrho>q$.
\end{proposition}

\begin{proof}
Before capping, selected and unselected weights are $1-\beta+\beta/p$ and $1-\beta$. Since $D>1$, only selected weights can be capped. Set $a=\min\{1-\beta+\beta/p,D\}$ and $b=1-\beta$. Adding the uniform rebate $1-pa-(1-p)b$ gives selected weight $1+(1-p)(a-b)$ and unselected weight $1-p(a-b)$. With $\eta=p(a-b)=\min\{\beta,p(D-1+\beta)\}$, these equal Eq.~(\ref{eq:span-rule}). Because $0\le\eta\le\beta\le1$, weights are nonnegative, and their sum is $L(1-\eta)+|S|\eta/p=L$. Multiplication by $A$ or $|A|$ proves conservation. Finally,
\[
\frac1L\sum_{t\in\mathcal E}d_t
=q(1-\eta)+\frac{|S\cap\mathcal E|}{L}\frac{\eta}{p}
=q+\eta(\varrho-q).
\]
Subtracting the uniform share $q$ proves the strict-improvement condition.
\end{proof}

Disabling SCR, empty or full masks, ineligible blocks, and zero coefficients leave allocation uniform; $\beta=0$ also gives uniform weights. The cap applies before the rebate, so final weights can exceed $D$; they satisfy $d_t\le D+\beta$. With the target set, mask size, and redistribution parameters fixed, a uniformly random mask has expected precision $q$ and expected target allocation $q$. SCR strictly improves on this expectation when $\eta>0$ and its mask is more precise than the target's prevalence.

For an eligible reference block matched in every valid rollout with identical scores $0\le r<c$, let $\lambda>0$. UAP supplies $A=-\lambda(c-r)$ while matched-only normalization supplies zero. Substituting this coefficient into Eq.~(\ref{eq:precision-gain}), SCR changes the target's absolute mass relative to uniform UAP by $L\lambda(c-r)\eta(\varrho-q)$. This change is strictly positive when $\eta>0$ and $\varrho>q$, linking feedback activation to localization under an informative mask. All conservation and allocation statements concern the source block before composition with page or omission contributions.

\clearpage
\begin{samepage}
\section{Qualitative Case Studies}
\label{app:case-studies}
\label{app:recovery-case}

Figures~\ref{fig:goodcase-omission-ejcts}--\ref{fig:goodcase-omission-chemistry} supplement Section~\ref{sec:credit-analysis} with selected OmniDocBench v1.6 outputs from separate GRPO and P2TPO runs after 100 training updates. Each figure places the source page on the left and reference excerpts, generated excerpts, and per-page edit distances (ED; lower is better) on the right. References come from the evaluator's ground-truth records, and generated excerpts come from saved model outputs. Ellipses indicate abridgment; italic annotations identify omitted content and are not generated text. These examples illustrate output differences without isolating individual credit-assignment components.
\par
\end{samepage}

\begin{figure}[H]
\centering
\begin{minipage}[t]{0.37\linewidth}
\vspace{0pt}
\includegraphics[width=\linewidth,height=0.48\textheight,keepaspectratio]{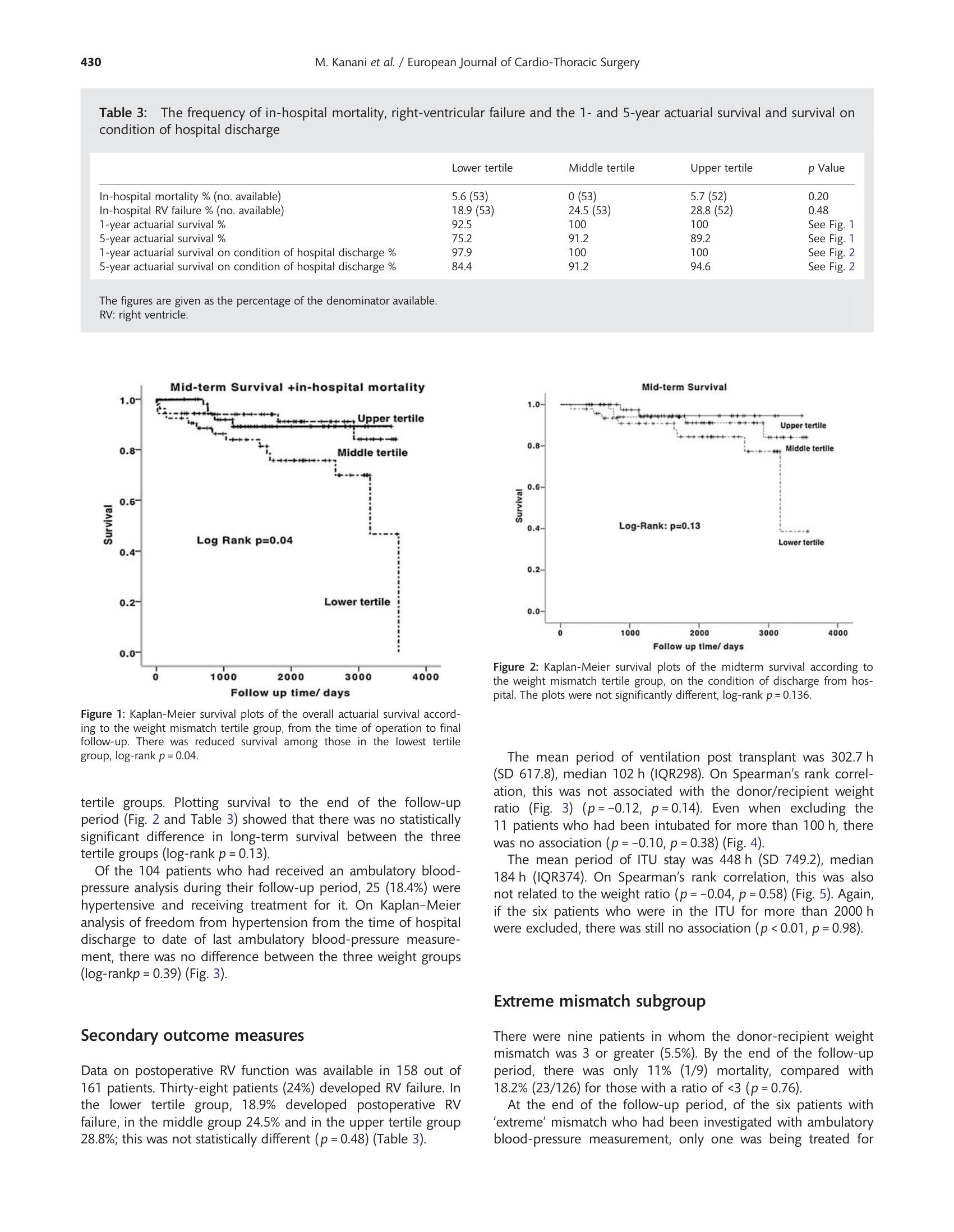}
\end{minipage}\hfill
\begin{minipage}[t]{0.59\linewidth}
\vspace{0pt}
\small\raggedright
\textbf{Clinical article: omission of a heading and paragraph}\par\smallskip
\textbf{Reference excerpt}\par
\texttt{Secondary outcome measures}\par
\texttt{Data on postoperative RV function was available in 158 out of 161 patients. Thirty-eight patients (24\%) developed RV failure. \ldots}\par\smallskip
\textbf{GRPO output excerpt}\par
\texttt{Figure 1: Kaplan-Meier survival plots \ldots}\par
\emph{The intervening text, including the heading and paragraph shown above, is absent.}\par
\texttt{Figure 2: Kaplan-Meier survival plots of the midterm survival \ldots}\par\smallskip
\textbf{P2TPO output excerpt}\par
\texttt{Secondary outcome measures}\par
\texttt{Data on postoperative RV function was available in 158 out of 161 patients. Thirty-eight patients (24\%) developed RV failure. \ldots}\par\smallskip
\textbf{Comparison.} GRPO skips the displayed heading and paragraph between two figure captions; P2TPO transcribes both. The table edit distance is unchanged.\par\smallskip
\setlength{\tabcolsep}{4pt}
\begin{tabularx}{\linewidth}{@{}Xrr@{}}
\toprule
Metric $\downarrow$ & GRPO & P2TPO \\
\midrule
Text ED & 0.4780 & 0.0000 \\
Reading-order ED & 0.5000 & 0.0000 \\
Table ED & 0.0663 & 0.0663 \\
\bottomrule
\end{tabularx}
\end{minipage}
\caption{Qualitative comparison of heading and paragraph transcription in a clinical article from OmniDocBench v1.6 ({\urlstyle{same}\nolinkurl{scihub_j.ejcts.2011.04.042.pdf_3}}).}
\label{fig:goodcase-omission-ejcts}
\end{figure}

\begin{figure}[H]
\centering
\begin{minipage}[t]{0.37\linewidth}
\vspace{0pt}
\includegraphics[width=\linewidth,height=0.48\textheight,keepaspectratio]{images/goodcase2_economist_p032.png}
\end{minipage}\hfill
\begin{minipage}[t]{0.59\linewidth}
\vspace{0pt}
\small\raggedright
\textbf{Magazine article: omission within a multi-column page}\par\smallskip
\textbf{Reference excerpt}\par
\texttt{It looks like the future}\par
\texttt{Along with Chennai, Tamil Nadu's capital, Bangalore and Hyderabad are also big destinations for foreign companies establishing ``global capability centres'' \ldots}\par\smallskip
\textbf{GRPO output excerpt}\par
\texttt{Monthly consumer spending per person\ldots}\par
\emph{The subheading and accompanying text shown above are absent.}\par
\texttt{The BJP's challenge is particularly acute in Tamil Nadu. \ldots}\par\smallskip
\textbf{P2TPO output excerpt}\par
\texttt{It looks like the future}\par
\texttt{Along with Chennai, Tamil Nadu's capital, Bangalore and Hyderabad are also big destinations for foreign companies establishing ``global capability centres'' \ldots}\par\smallskip
\textbf{Comparison.} GRPO omits the displayed subheading and accompanying text, resuming with the paragraph beginning ``The BJP's challenge.'' P2TPO retains the passage.\par\smallskip
\setlength{\tabcolsep}{4pt}
\begin{tabularx}{\linewidth}{@{}Xrr@{}}
\toprule
Metric $\downarrow$ & GRPO & P2TPO \\
\midrule
Text ED & 0.2721 & 0.0153 \\
Reading-order ED & 0.4286 & 0.0000 \\
\bottomrule
\end{tabularx}
\end{minipage}
\caption{Qualitative comparison of text coverage in a multi-column magazine page from OmniDocBench v1.6 ({\urlstyle{same}\nolinkurl{magazine_TheEconomist.2024.03.02_page_032}}).}
\label{fig:goodcase-omission-economist}
\end{figure}

\begin{figure}[H]
\centering
\begin{minipage}[t]{0.37\linewidth}
\vspace{0pt}
\includegraphics[width=\linewidth,height=0.48\textheight,keepaspectratio]{images/goodcase3_s00396_p2.png}
\end{minipage}\hfill
\begin{minipage}[t]{0.59\linewidth}
\vspace{0pt}
\small\raggedright
\textbf{Chemistry article: omissions and mathematical notation}\par\smallskip
\textbf{Reference excerpt}\par
\texttt{Measurements}\par
\texttt{Elemental analysis was performed with a Elementar-CHNO-Rapiol to determine the carbon, nitrogen, and hydrogen content. \ldots}\par\smallskip
\textbf{GRPO output excerpt}\par
\texttt{2,855 cm -1 . 1 H NMR(CDCl 3 ): \ldots}\par
\emph{The ``Measurements'' heading and elemental-analysis paragraph are absent.}\par\smallskip
\textbf{P2TPO output excerpts}\par
\texttt{2,855 cm\$\^{}\{-1\}\$. \ldots}\par
\texttt{Measurements}\par
\texttt{Elemental analysis was performed with a Elementar-CHNO-Rapiol to determine the carbon, nitrogen, and hydrogen content. \ldots}\par\smallskip
\textbf{Comparison.} P2TPO retains the displayed ``Measurements'' heading and paragraph and uses inline \LaTeX{} notation for the superscript. The corresponding GRPO excerpt omits the passage and renders the superscript as plain text. The per-page metrics below show the accompanying text and reading-order differences and retain the table result for context.\par\smallskip
\setlength{\tabcolsep}{4pt}
\begin{tabularx}{\linewidth}{@{}Xrr@{}}
\toprule
Metric $\downarrow$ & GRPO & P2TPO \\
\midrule
Text ED & 0.4625 & 0.0427 \\
Reading-order ED & 0.5833 & 0.0000 \\
Table ED & 0.0734 & 0.1152 \\
\bottomrule
\end{tabularx}
\end{minipage}
\caption{Qualitative comparison of text coverage and superscript notation in a chemistry article from OmniDocBench v1.6 ({\urlstyle{same}\nolinkurl{docstructbench_llm-raw-scihub-o.O-s00396-003-0935-9.pdf_2}}).}
\label{fig:goodcase-omission-chemistry}
\end{figure}

\end{CJK*}
\end{document}